\documentclass{article}

\PassOptionsToPackage{numbers,sort&compress}{natbib}

\usepackage[preprint]{neurips_2026}
\usepackage{amsmath,amssymb,amsthm}
\usepackage{mathtools}
\usepackage{booktabs}
\usepackage{microtype}
\usepackage{graphicx}
\usepackage{hyperref}
\usepackage{cleveref}
\usepackage{xcolor}
\usepackage{xspace}
\usepackage{bbm}
\usepackage{tcolorbox}
\tcbuselibrary{breakable, skins}

\tcbset{
  promptbox/.style={
    enhanced,
    breakable,
    colback=gray!7,
    colframe=gray!40,
    boxrule=0.4pt,
    arc=3pt,
    left=6pt, right=6pt, top=4pt, bottom=4pt,
    fontupper=\small\ttfamily,
    before upper={\setlength{\parskip}{2pt}},
  }
}

\newcommand{\bottomtablecaption}[1]{\vspace{5pt}\caption{#1}}

\theoremstyle{plain}
\newtheorem{theorem}{Theorem}[section]
\newtheorem{proposition}[theorem]{Proposition}

\newtheorem*{theoremA}{Theorem A}
\newtheorem*{theoremB}{Theorem B}
\makeatletter
\newcommand{\labeltheoremA}[1]{%
  \begingroup
  \def\@currentlabel{A}%
  \def\@currentlabelname{Theorem A}%
  \label{#1}%
  \endgroup
}
\newcommand{\labeltheoremB}[1]{%
  \begingroup
  \def\@currentlabel{B}%
  \def\@currentlabelname{Theorem B}%
  \label{#1}%
  \endgroup
}
\makeatother

\theoremstyle{definition}

\newtheorem{assumption}{Assumption}

\theoremstyle{remark}
\newtheorem{remark}[theorem]{Remark}

\newcommand{\Hcal}{\mathcal{H}}
\newcommand{\Gcal}{\mathcal{G}}
\newcommand{\Pp}{\mathbb{P}^{\pi}}
\newcommand{\Ep}{\mathbb{E}_{\mathbb{P}^{\pi}}}

\DeclareRobustCommand{\TPS}{\ensuremath{\mathrm{TPS}}\xspace}
\DeclareRobustCommand{\TPSlog}{\ensuremath{\mathrm{TPS}_{\log}}\xspace}
\DeclareRobustCommand{\TPSBrier}{\ensuremath{\mathrm{TPS}_{\mathrm{Brier}}}\xspace}
\DeclareRobustCommand{\TPSBeta}[2]{\ensuremath{\mathrm{TPS}_{\beta(#1,#2)}}}
\DeclareRobustCommand{\TPSobs}{\ensuremath{\mathrm{TPS}^{\mathrm{obs}}}\xspace}
\DeclareRobustCommand{\TPScenabs}{\ensuremath{\mathrm{TPS}^{\mathrm{cen,abs}}}\xspace}
\DeclareRobustCommand{\TPScenexact}{\ensuremath{\mathrm{TPS}^{\mathrm{cen,exact}}}\xspace}
\DeclareRobustCommand{\TPScensimple}{\ensuremath{\mathrm{TPS}^{\mathrm{cen,simple}}}\xspace}

\newcommand{\etal}{\textit{et al.}\xspace}

\title{Proper Scoring Rules for Agentic Uncertainty Quantification}
 
\author{%
  Suresh Raghu\thanks{Equal contribution.} \\
  Independent Researcher \\
  \texttt{sureshraghu0706@gmail.com} \\
  \And
  Satwik Pandey\footnotemark[1] \\
  Independent Researcher \\
  \texttt{psatwik2711@gmail.com} \\
  \And
  Shashwat Pandey \\
  Independent Researcher \\
  \texttt{shashwat1225@gmail.com} \\
}

\begin{document}
\maketitle

\begin{abstract}
Language-model agents increasingly emit uncertainty signals
throughout a trajectory, but existing agentic UQ evaluations often conflate
ranking usefulness with probabilistic truthfulness. AUROC, AUPRC,
risk--coverage, Trajectory ECE, and scalarized trajectory scores evaluate
discrimination, binwise calibration, or collapsed summaries, but do not
strictly elicit the full prefix-conditioned success-probability trace
$q_t=\mathbb{P}^{\pi}(Y{=}1\mid\mathcal{H}_t)$. Building on prequential
proper scoring, we introduce the Trajectory Proper Score (\TPS), a
predictor-agnostic family of strictly proper trajectory-level scoring rules
for any per-step uncertainty signal calibrated into a probability of eventual
success. We prove that \TPS strictly elicits the success-probability process
under complete observation, within the chosen score family and weight
schedule. We extend the construction to administratively censored
trajectories by projecting the complete-data score onto the observable stopped
prefix, yielding an exact $q_Z$-weighted reduced score and a tractable
approximation when $q_Z$ is unestimated. We further show that common trajectory
evaluators target weaker objects than the full prefix-conditioned probability
process: Trajectory ECE is resolution-blind, while scalarized Trajectory
Brier elicits only the collapsed scalar, not the full trace. Experiments on
StrategyQA, Tau2-Bench, HotpotQA, and WebShop show that these theoretical
distinctions are operationally visible: probability recalibration can
substantially change \TPS while leaving rank metrics nearly unchanged, and the
tractable censored approximation can change the verdict relative to
complete-only evaluation.
\end{abstract}

\section{Introduction}
\label{sec:intro}
 
Language-model agents increasingly act through extended trajectories:
they observe, reason, call tools, revise plans, and eventually either
succeed or fail. In such settings, uncertainty is not a property of the
final answer alone. A useful agent should know, at each prefix of the
interaction, how likely the task is to succeed from that point onward.


Recent agentic UQ work has moved from final-answer uncertainty toward trajectory-level signals, including stepwise calibration, propagated uncertainty, and trajectory confidence aggregation \cite{Zhao2024,Duan2025,Zhang2026,Wang2025STeCa}. These methods mark an important shift, but they leave open a more basic evaluator-side question: what should count as a correct evaluation of such a trace?

We argue that the missing object is a strictly proper evaluator for the
prefix-conditioned probability process. Existing trajectory evaluations
typically rely on AUROC, trajectory-ECE, or scalarized trajectory-Brier.
These are useful diagnostics for particular summaries: AUROC measures
rank discrimination, ECE measures binwise marginal calibration, and
scalar Brier scores a collapsed binary probability. None of them,
however, strictly elicits the full process
$(F_t)_{t=1}^{T}$, where $F_t$ is the agent's reported probability of
eventual task success from the current prefix. This distinction matters:
two agents can induce similar rankings while assigning very different
probabilities, and downstream interventions such as deferral, reflection,
or human handoff depend on the probability scale, not only on rank.

The forecast object in this paper is deliberately simple. At step $t$,
after observing history $\mathcal H_t$, the agent reports
$F_t \in [0,1]$, interpreted as the probability that the task will
eventually succeed. The truthful target is
\[
q_t := \mathbb P^\pi(Y=1 \mid \mathcal H_t),
\]
the continuation-success probability under the evaluation law. This process
need not be monotone: a calibrated agent should lower $F_t$ after a harmful
action and raise it after acquiring useful evidence~\cite{Kirchhof2025}.
 
We make the following contributions.
\begin{enumerate}
  \item We formalize agentic uncertainty evaluation as strictly proper
      elicitation of the prefix-conditioned success process
      \(q_t=\mathbb P^\pi(Y=1\mid \mathcal H_t)\), and define the
      weighted trajectory score
      \[
      \TPS(F_{1:T},Y)
      =
      \sum_{t=1}^{T} w_t\, S_{\alpha,\beta}(F_t,Y).
      \]

  \item We extend the trajectory score to censored traces via conditional
        projection onto the observable prefix. In the binary agentic setting
        this reduces to an exact censored score parameterized by a single
        continuation-success weight $q_Z$.

  \item We show that common trajectory evaluators target weaker objects:
      trajectory-ECE is resolution-blind, and scalarized trajectory-Brier is
      proper only for a collapsed scalar rather than the full
      prefix-conditioned trace.

  \item Empirically, we show that recalibration can leave rank-based metrics
      nearly unchanged while strictly proper trajectory scores reveal large
      probability-scale errors, and that censored-aware scoring yields
      different conclusions from complete-only evaluation on administratively
      censored agent trajectories.
\end{enumerate}

The aim of this work is to separate the question of how an agent
produces uncertainty from the question of whether the reported
uncertainty means what it says. \TPS is agnostic to the
source of the signal: verbal confidence, token-probability features,
post-hoc calibrators, and predictor-side agentic UQ methods such as
SAUP, UProp, and AUQ can all be evaluated once expressed as
prefix-conditioned success probabilities.
 
\section{Related Work}
\label{sec:related}

Strictly proper scoring rules provide the classical language for evaluating
probabilistic forecasts. A proper score rewards truthful probability reporting,
and a strictly proper score makes the truthful report the unique optimum;
Gneiting and Raftery~\cite{Gneiting2007} characterize regular proper scores as
those generated by a convex function. In the binary case, proper scores admit
threshold-mixture representations~\cite{Schervish1989}; the beta family
\(S_{\alpha,\beta}\) uses this view to recover log and Brier scores as
special cases while allowing bounded asymmetric cost shaping~\cite{Buja2005}.
Calibration--resolution--uncertainty decompositions further show why marginal
calibration is not enough: a forecast can be reliable on average while carrying
little resolution about the outcome~\cite{Murphy1973,Degroot1983,Brocker2009}.
This distinction is central for agentic UQ, where an evaluator should reward
informative prefix-conditioned probabilities rather than only binwise
calibration. Our complete-data trajectory score applies this binary
scoring-rule machinery at each prefix and aggregates the resulting scores with
positive weights.

Proper scoring under censoring has been studied mainly in survival analysis,
where right-censored event times require observable-data extensions of
complete-data scores. Rindt \etal~\cite{Rindt2022} prove that a censored
log-score extension is strictly proper for right-censored survival times,
while Blanche \etal~\cite{Blanche2019} show that common survival metrics,
including C-index variants and integrated Brier-style scores, can fail
propriety. Yanagisawa~\cite{Yanagisawa2023} systematizes this view by
extending multiple proper scoring rules to censored observations and
identifying when the required censoring weights are valid. Our censored
extension borrows this conditional-projection view, but targets a different
forecast object: \(F_t=\Pr(Y=1\mid\mathcal H_t)\), the probability of one
fixed terminal outcome from the current agent prefix, rather than a cumulative
event-time distribution. Thus the trace may rise or fall as evidence,
mistakes, or recovery occur, unlike the monotone CDF target in standard
survival analysis.

Agentic uncertainty quantification has so far focused mainly on producing
uncertainty signals. Local predictors include verbalized confidence
\cite{Han2024}, token-level measures such as log-probability and predictive
entropy~\cite{Malinin2021,Duan2024}, and semantic entropy over sampled
responses~\cite{Kuhn2023}. More recent agentic methods move from local
signals to trajectory-level uncertainty: SAUP propagates step uncertainty with
situation-aware weights~\cite{Zhao2024}, UProp studies how uncertainty is
inherited across multi-step decision chains~\cite{Duan2025}, AUQ evaluates
trajectory confidence aggregators such as
\(\Phi_{\mathrm{last}}\), \(\Phi_{\mathrm{avg}}\), and
\(\Phi_{\min}\)~\cite{Zhang2026}, and STeCa uses step-level calibration
signals for agent learning~\cite{Wang2025STeCa}. These methods describe how
uncertainty signals are produced, propagated, or used; we ask how to evaluate
them once expressed as prefix-conditioned success probabilities.

Evaluation practice has not kept pace with the predictor side. UProp evaluates
uncertainty with AUROC and selective-prediction metrics such as AUARC, while
AUQ evaluates trajectory-level confidence aggregators with trajectory-ECE and
trajectory-Brier. These metrics are useful diagnostics, but they target weaker
objects: AUROC and AUPRC measure ranking, AUARC/AURC-style metrics measure
selective-prediction ordering, ECE measures binwise marginal calibration of a
scalar aggregate, and scalarized Brier is proper only for the aggregate
confidence value supplied to it. Static-classification work has already shown
that ECE is sensitive to binning and can obscure calibration structure
\cite{Kumar2020,Vaicenavicius2019}. Recent risk-control work is complementary:
it refines decision-control diagnostics, whereas \(\TPS\) targets strict
elicitation of the full prefix-conditioned probability trace
\cite{traub2024overcomingcommonflawsevaluation}. In agentic settings, the
issue is sharper because the trajectory is first aggregated before the
diagnostic is applied. We formalize this gap in Theorems~A and~B.
 
\section{Preliminaries}
\label{sec:prelim}

An agent executes a task over up to $T$ steps. At each step $t$, it observes
$o_t$, takes action $a_t$, and emits a scalar forecast $F_t\in[0,1]$. The
prefix history is
\[
h_t=(o_0,a_0,\ldots,o_t),
\]
and $\Hcal_t:=\sigma(H_t)$ is the natural history filtration,
$\Hcal_1\subseteq\cdots\subseteq\Hcal_T$. A valid step-$t$ forecast is
$\Hcal_t$-measurable. This filtration view is consistent with the standard
partially observable Markov decision process (POMDP) framing of agent
interaction \cite{Kaelbling1998}: $F_t$ can be viewed as the
agent's belief state marginalized onto the binary success/failure partition.
This is a prequential forecasting view~\citep{Dawid1984}: forecasts are
evaluated sequentially as information accumulates, but here each forecast
targets the same eventual binary outcome rather than a new observation at
each time step.

The forecast sequence is not a cumulative distribution function and need
not be monotone: a calibrated agent may lower its forecast after a harmful
action and raise it after acquiring useful evidence.

Fix an evaluation law $\mathbb P^\pi$ over full trajectories, induced by
the benchmark distribution, environment randomness, tool randomness, and
the fixed evaluation policy $\pi$. All conditional probabilities and
expectations below are taken under this law. For readability, we suppress
the superscript $\pi$ and write $\mathbb P$ and $\mathbb E$ unless policy
dependence is important. Let $Y\in\{0,1\}$ denote the terminal task outcome.
The truthful continuation-success process is
\begin{equation}
  q_t := \mathbb P(Y{=}1\mid \Hcal_t)
       = \mathbb E[Y\mid \Hcal_t].
  \label{eq:qt}
\end{equation}
Equivalently, $(q_t)_{t=1}^T$ is the martingale of the terminal outcome
with respect to the history filtration: its current value is the best
conditional prediction of the final outcome given the current prefix.
Operationally, $q_t$ is the continuation-success frequency one would obtain
by re-rolling many completions from the same prefix under the same evaluation
law. It is not the agent's reported confidence; it is defined by
$\mathbb P$. Under complete observation, $Y$ is observed at termination and
$Y_t=Y$ for all $t$: the target event does not change, only the information
available about it.
  
For censored trajectories, let $T^*$ denote the step at which the task
outcome becomes determined and let $C$ denote the exogenous termination
step, such as a fixed step budget or compute budget. Let $X$ denote the
task description or instance, which we take to be included in the initial
history $H_0$. We observe
\[
Z=\min(T^*,C), \qquad \Delta=\mathbf 1(T^*\le C).
\]
We write $z,\delta$ for realized values of $Z,\Delta$. If
$\Delta=1$, the trajectory is complete and $Y$ is observed. If
$\Delta=0$, the trajectory is administratively censored at prefix $Z$ and
$Y$ is unobserved.

Throughout, $H_t$ denotes the random history at step $t$, while $h_t$
denotes a realized history. The stopped history is denoted $H_Z$, and
$F_{1:Z}:=(F_1,\ldots,F_Z)$ denotes the forecast prefix observed before
termination. We write \(q_Z=\mathbb P^\pi(Y=1\mid\mathcal H_Z)\) for the
stopped random variable and \(q_z\) for its realized value on a trajectory
censored at step \(z\).

\begin{assumption}[Non-informative censoring]
\label{asm:noninformative}
\[
T^* \perp C \mid X .
\]
\end{assumption}

Assumption~\ref{asm:noninformative} requires that, conditional on the task
instance, the external termination mechanism is independent of when the
trajectory outcome would have been determined. This covers fixed step
budgets and fixed compute budgets. It is violated by adaptive
monitor-driven stopping, where a trajectory is stopped precisely because it
appears likely to fail.

\begin{assumption}[Administrative stop]
\label{asm:admin}
\[
Y \perp (C,\Delta)\mid \Hcal_Z .
\]
\end{assumption}

Assumption~\ref{asm:admin} requires that, once the stopped prefix is known,
the fact that administrative censoring occurred at step $Z$ carries no
additional information about the eventual outcome. In other words, the
observable prefix may be informative about success or failure, but the
administrative stopping event itself is not.

We use reward-oriented scoring rules, so larger scores are better. A binary
scoring rule \(S(p,y)\) is proper if, conditionally on \(\Hcal_t\),
\[
  \Ep[S(q_t,Y)\mid \Hcal_t]
  \ge
  \Ep[S(p,Y)\mid \Hcal_t]
\]
for every \(\Hcal_t\)-measurable forecast \(p\), and strictly proper if
equality implies \(p=q_t\) almost surely. We instantiate \(S\) with the
beta-family scores \(S_{\alpha,\beta}\) derived from the Schervish--Buja
threshold-mixture construction~\cite{Schervish1989,Buja2005}, using mixing
measure \(\nu(dc)=c^{\alpha-1}(1-c)^{\beta-1}\,dc\). The family is strictly
proper for all \(\alpha,\beta>0\) and includes Brier (up to positive affine
equivalence) and the log score as a limiting strictly proper endpoint;
asymmetric choices \(\alpha\neq\beta\) provide cost-shaping. Explicit
formulas, boundary behavior, and the proof are in Appendix~\ref{app:beta}.

\section{Proper Scoring Rules for Agentic UQ}
\label{sec:theory}

\subsection{Trajectory score under complete observation}
\label{sec:uncensored}
\label{sec:tps_complete}

Let $w_{1:T}$ be a fixed evaluator-chosen weight schedule over trajectory
steps, with $w_t>0$ and, for reporting convenience,
$\sum_{t=1}^{T}w_t=1$. These weights are not uncertainty predictions; they
encode which prefixes the evaluator wishes to emphasize. Uniform weights score every prefix equally, while front-loaded weights
emphasize early prefixes, where miscalibration can shape later actions,
observations, and opportunities for recovery~\cite{Dziri2023,Wang2025STeCa}.

We define the \emph{Trajectory Proper Score} (\TPS) as the weighted lift of a
strictly proper binary score to the prefix-conditioned forecast process:
\begin{equation}
  \TPS(F_{1:T},Y)
  =
  \sum_{t=1}^{T} w_t\,S_{\alpha,\beta}(F_t,Y).
  \label{eq:tps}
\end{equation}
The name is meant to emphasize that this is a family of trajectory-level
proper scores rather than a single fixed scoring rule: any strictly proper
binary score \(S_{\alpha,\beta}\), combined with any fixed positive weight
schedule \(w_t\), yields a strictly proper trajectory evaluator. We use
subscripts to denote particular members, such as \TPSlog, \TPSBrier, and
\TPSBeta{2}{4}.
Our default instantiation is the normalized linear front-loaded schedule
\[
w_t = \frac{2(T-t+1)}{T(T+1)}.
\]
For variable-length trajectories, \(T_i\) denotes the evaluated horizon of
trajectory \(i\): the realized terminal length for complete trajectories and
the administrative budget/observed stop length for naturally censored
max-step trajectories. The schedule \(w_{it}\) is constructed separately
within each \(T_i\) and normalized over \(t=1,\ldots,T_i\); in artificial
censoring, the original \(T_i\)-normalized weights are truncated at \(Z_i\)
and are not renormalized over the observed prefix.
Uniform weighting is included as a robustness check in the experiments. The
strict-propriety result below does not depend on this particular schedule;
it requires only that the weights are fixed exogenously and strictly
positive.

\begin{theorem}[Strict propriety under complete observation]
\label{thm:complete}
If $S_{\alpha,\beta}$ is strictly proper and $w_t>0$ for all $t$, then
\TPS is strictly proper for the prefix-conditioned success process:
\[
\Ep[\TPS(q_{1:T},Y)]
\ge
\Ep[\TPS(F_{1:T},Y)],
\]
with equality if and only if $F_t=q_t$ a.s.\ for every $t$.
\end{theorem}

\noindent
The proof is by conditional strict propriety at each filtration level and
summing nonnegative regrets with positive weights; see
Appendix~\ref{app:complete-proof}.

In our experiments we report log, Brier, and Beta(2,4) as representative
members of the same proper family: log is the canonical unbounded default,
Brier is the bounded symmetric baseline, and Beta(2,4) is a bounded
asymmetric member.  Appendix~\ref{app:beta} also gives their cost-shaping
interpretation.

\subsection{Censoring as conditional projection}
\label{sec:abstract_censored}

We now extend \TPS to trajectories whose final outcome is not observed. Let
\[
\Gcal_Z := \sigma(H_Z,Z,\Delta,\Delta Y)
\]
be the sigma-field generated by the stopped trajectory, the censoring
indicator, and the observed terminal outcome on complete trajectories.
The term $\Delta Y$ records $Y$ when $\Delta=1$ and contributes no
outcome information when $\Delta=0$. If $Y$ is already included in
$H_Z$ on complete trajectories, this term is redundant.

For a stopped prefix ending at $Z$, define the observable-prefix version of
\TPS by
\[
{\TPSobs}_{\alpha,\beta}(F_{1:Z},Y;Z)
:=
\sum_{t=1}^{Z} w_t\,S_{\alpha,\beta}(F_t,Y).
\]
The weights $w_t$ are inherited from the full trajectory score and are not
renormalized over the observed prefix; this preserves comparability between
complete-only and censored-prefix scores.
 
The abstract censored TPS is the conditional expectation of this
observable-prefix complete-data score given the stopped trajectory:
\begin{equation}
  {\TPScenabs}_{\alpha,\beta}(F_{1:Z};\Gcal_Z)
  :=
  \Ep\!\left[
    {\TPSobs}_{\alpha,\beta}(F_{1:Z},Y;Z)
    \mid \Gcal_Z
  \right].
  \label{eq:abs_cen}
\end{equation}

\begin{theorem}[Abstract censored propriety on the observable prefix]
\label{thm:abstract_censored}
If the per-step score $S_{\alpha,\beta}$ is proper under complete
observation, then \({\TPScenabs}_{\alpha,\beta}\) is proper for the observable
prefix $F_{1:Z}$ under the stopped-data law. Strictness, when available, is
restricted to the observable prefix.
\end{theorem}

\noindent
The proof is the conditional-expectation projection/tower-property argument;
see Appendix~\ref{app:censored-proof}.

\subsection{Exact reduced censored beta score}
\label{sec:cen_exact}

The preceding reduction yields a directly interpretable censored extension of
\TPS in the oracle case where the continuation-success weight $q_Z$ is
available. We write \TPScenexact and \TPScensimple for censored
extensions of \TPS: the superscript indicates how the unobserved outcome is
handled, while the subscript continues to indicate the binary score family.
We define
\begin{equation}
  \begin{aligned}
  \mathrm{TPS}^{\mathrm{cen,exact}}_{\alpha,\beta}(F_{1:Z};\Gcal_Z)
  =
  \sum_{t=1}^{Z} w_t
  \Bigl[
    &\Delta\,S_{\alpha,\beta}(F_t,Y) \\
    &+
    (1-\Delta)
    \Bigl(
      q_Z S_{\alpha,\beta}(F_t,1)
      +
      (1-q_Z)S_{\alpha,\beta}(F_t,0)
    \Bigr)
  \Bigr].
  \end{aligned}
  \label{eq:cen_exact}
\end{equation}
Here $q_Z:=\Pp(Y=1\mid\Hcal_Z)$ is the stopped
continuation-success process defined in Section~\ref{sec:abstract_censored}.

\begin{theorem}[Exact reduced censored beta score]
\label{thm:cen_exact}
Under Assumptions~\ref{asm:noninformative}--\ref{asm:admin},
\(\mathrm{TPS}^{\mathrm{cen,exact}}_{\alpha,\beta}\) is proper for every proper beta-family member.
If $S_{\alpha,\beta}$ is strictly proper and $w_t>0$ for all observed
steps, then \(\mathrm{TPS}^{\mathrm{cen,exact}}_{\alpha,\beta}\) is strictly proper on the
observable prefix $F_{1:Z}$.
\end{theorem}

\noindent
Appendix~\ref{app:censored-proof} proves the explicit reduced form, which
reduces at the logarithmic endpoint to the corresponding soft-label log score
on censored prefixes. Appendix~\ref{app:qz:hotpotqa} validates this exact
\(q_Z\)-weighted construction with Monte Carlo continuation rollouts.
 
\subsection{Simple censored score for operational evaluation}
\label{sec:cen_simple}

The exact reduced score requires the continuation-success weight $q_Z$ for
censored prefixes. When this weight is not estimated, we use the pessimistic
operational approximation
\[
q_Z \approx 0,
\]
which scores a censored prefix through the failure-side branch of the
binary score. Equivalently, this sets the censored posterior success weight
in~\eqref{eq:cen_exact} to zero:
\begin{equation}
  \mathrm{TPS}^{\mathrm{cen,simple}}_{\alpha,\beta}(F_{1:Z};\Gcal_Z)
  =
  \sum_{t=1}^{Z} w_t
  \left[
    \Delta\,S_{\alpha,\beta}(F_t,Y)
    +
    (1-\Delta)\,S_{\alpha,\beta}(F_t,0)
  \right].
  \label{eq:cen_simple}
\end{equation}
At the logarithmic endpoint, the same failure-side reduction gives the
simple censored-log score used as our primary operational censored
evaluator:
\begin{equation}
  \mathrm{TPS}^{\mathrm{cen,simple}}_{\log}(F_{1:Z};\Gcal_Z)
  =
  \sum_{t=1}^{Z} w_t
  \left[
    \Delta\{Y\log F_t+(1-Y)\log(1-F_t)\}
    +
    (1-\Delta)\log(1-F_t)
  \right].
  \label{eq:cen_simple_log}
\end{equation}

When interpreted as a proper score, \(\TPScensimple\) elicits the
pseudo-label target
\[
  m_t=\Ep[\Delta Y\mid\Hcal_t],
\]
not the original continuation-success target
\(q_t=\Pp(Y=1\mid\Hcal_t)\), unless the missing success mass is zero.
We therefore use \(\TPScensimple\) only as an explicit pessimistic
\(q_Z\approx0\) approximation to the exact reduced score. The logarithmic
instance is our primary operational censored evaluator on WebShop; Brier and
Beta-family instances are reported only as robustness checks.
Appendix~\ref{app:censored-proof} gives the pseudo-label target result and
proof.

\section{Evaluation Metrics for Agentic UQ}
\label{sec:metrics}

\subsection{Why common trajectory evaluators are insufficient}
\label{sec:neg_results}
 
We now organize common agentic UQ evaluators by the forecast object they
elicit: a rank ordering, a binwise-calibrated scalar, a collapsed
probability, or the full prefix-conditioned probability process.
Formally, Theorem~\ref{thm:A} shows that T-ECE can tie non-truthful
resolution-losing forecasts with the truthful process, while
Theorem~\ref{thm:B} shows that scalarized proper losses elicit only
\(C=\Phi(F_{1:T})\), not the full trace.
 
\noindent\textbf{AUROC and rank metrics.}
AUROC evaluates whether a scalar trajectory score ranks successful and
unsuccessful trajectories correctly. It is invariant under any strictly
monotone reparametrization of the scalar score
\(C=\Phi(F_{1:T})\) supplied to it~\cite{Hanley1982}. Thus AUROC can
distinguish useful ordering from useless ordering, but it cannot distinguish
well-scaled probabilities from overconfident or underconfident monotone
distortions. The same issue applies to other rank- or threshold-based
diagnostics such as AUPRC and risk--coverage/AURC: they are useful
discrimination diagnostics, but they do not elicit the probability trace.
 
\noindent\textbf{Trajectory ECE.}
Trajectory ECE~\cite{Zhang2026} collapses the forecast trace
to a scalar \(C=\Phi(F_{1:T})\) and checks whether empirical success rates
match average confidence within fixed bins. A forecast that is right on average
within every bin scores zero T-ECE, regardless of whether it separates
successful from failed trajectories within those bins. This distinction
matters: a proper forecast must not only be calibrated on average
(\emph{Reliability}) but also concentrate probability mass where outcomes
differ (\emph{Resolution}). T-ECE measures only the first component of this
calibration--resolution--uncertainty decomposition
\cite{Murphy1973,Brocker2009,Degroot1983} and is blind to the
second. Concretely, the truthful process \(q\) and its binwise average
\(G=\Ep[q\mid\operatorname{bin}(q)]\) can both achieve T-ECE\(=0\), even
though \(G\) discards all within-bin discriminative information.

\noindent\textbf{Trajectory Brier and scalarized proper scores.}
Brier and log loss are strictly proper for a single binary probability
forecast. When a trajectory is first collapsed to
\(C=\Phi(F_{1:T})\) and the proper score is applied only to \(C\), however,
the resulting scalarized score elicits the collapsed scalar, not the full
trace. Common AUQ-style aggregators, including \(\Phi_{\mathrm{last}}\),
\(\Phi_{\mathrm{avg}}\), \(\Phi_{\min}\), and weighted averages, can
therefore hide or reward non-truthful intermediate forecasts.
The scalarization result is sharper:
\(\Phi_{\mathrm{last}}\)-Brier is proper but not strictly proper for the
trace, while \(\Phi_{\mathrm{avg}}\)-, \(\Phi_{\min}\)-, and
\(\Phi_w\)-Brier can make the truthful trace strictly suboptimal.

Each evaluator above is well suited to what it measures: AUROC, AUPRC,
and AURC for rank discrimination; T-ECE for binwise marginal calibration;
and scalarized Brier/log scores for a collapsed scalar probability. The
gap is specific: none of them strictly elicits the full prefix-conditioned
success-probability process \((q_t)\). \(\TPS\) fills this gap under
complete observation; \(\mathrm{TPS}^{\mathrm{cen,exact}}_{\alpha,\beta}\) extends it to the
observable prefix under administrative censoring.

\subsection{Reporting convention}
\label{sec:reporting}
 
These diagnostics also have no native score for trajectories with
unobserved \(Y\): censored traces must be discarded, continued, or
mislabeled before they can be applied. On long-horizon tasks this can
mean evaluating only the cases the agent completes, precisely where
miscalibration is least likely to surface. A complete agentic UQ evaluation
reports \(\TPS\) for complete trajectories, or
\(\mathrm{TPS}^{\mathrm{cen,exact}}_{\alpha,\beta}\) (the exact \(q_Z\)-weighted reduced censored
score) for censored prefixes when \(q_Z\) is available, as the primary proper
score. When \(q_Z\) is not estimated, \(\mathrm{TPS}^{\mathrm{cen,simple}}_{\log}\) (the log endpoint
of \(\TPScensimple\)) is reported as an operational censored-log approximation.
AUROC, AUPRC, AURC, and T-ECE are reported as diagnostic
complements. Predictor calibration status, weight schedule, and, for
censored data, the censoring assumption and exclusion criteria should
always be disclosed.

\section{Experiments}
\label{sec:experiments}

The preceding sections show that common trajectory diagnostics and
\(\TPS\) answer different evaluation questions. We now ask whether this
difference is large enough to change empirical verdicts on real agent
traces. We use these predictor streams to probe evaluator behavior, not to rank uncertainty methods.

%

\noindent\textbf{Setup.}
Experiments use Gemma 4 31B in a fixed ReAct harness on
Tau2-Bench, StrategyQA, HotpotQA, and WebShop. We evaluate five per-step
probability streams: verbal confidence, completion-token probability,
completion-entropy confidence, action-span token probability, and action-span
entropy confidence, plus a base-rate reference. Primary analyses use log score
with linear-front weights and cross-fitted Platt calibration where reported.
Predictor formulas and transparency tables are in Appendix~\ref{app:transparency};
calibration is in Appendix~\ref{app:calibration}; preprocessing, prompts, and
hyperparameters are in Appendix~\ref{app:implementation}; and robustness sweeps
are in Appendix~\ref{app:robustness}.

\subsection{Calibration-invariance gap}
\label{sec:exp2a}

We test whether rank-based evaluation and proper trajectory scoring remain
aligned when a forecast stream is recalibrated. We take the model's verbal
confidences on Tau2-Bench, which are heavily saturated near
\(\{0.90,0.95,1.00\}\), and compare the raw stream with a Platt-recalibrated
version. The base-rate stream
\(F_t\equiv\bar y\) serves as an uninformative reference. If AUROC and
\(\TPSlog\) (the log endpoint of \(\TPS\), i.e.\ \(\alpha=\beta=0\)) led to
the same practical verdict, recalibration should not separate them by many
bootstrap standard errors.

Figure~\ref{fig:exp2a} shows the opposite. Let \(\Delta/\mathrm{SE}\)
denote the raw-to-Platt change divided by its paired bootstrap standard
error. Recalibration changes AUROC by only \(-0.010\)
(\(\Delta/\mathrm{SE}\approx0.3\)), while improving \(\TPSlog\) by
\(5.67\) nats (\(\Delta/\mathrm{SE}\approx43\)). The corresponding AUPRC
and AURC changes are also small relative to the \(\TPSlog\) shift. Under \(\TPSlog\), the
raw stream lies \(5.66\) nats below the base-rate reference; the Platt
stream lies \(0.008\) nats above it.
Appendix~\ref{app:fixed_rank} repeats the comparison with the rank-metric
input held fixed by construction; AUROC/AUPRC/AURC are identical across
transforms while \(\TPSlog\) spans 5.7 nats.
 
After calibration, verbal confidence barely exceeds the base-rate reference, improving probability scale but not discrimination. AUROC 
treats the raw and recalibrated streams as nearly equivalent because 
their rankings are nearly equivalent, while \(\TPSlog\) detects that one
stream is badly mis-scaled as a probability forecast. This distinction 
matters whenever \(F_t\) is consumed as a probability for deferral, 
thresholding, human handoff, or combination with a cost model.

\begin{figure}[t]
  \centering
  \includegraphics[width=\linewidth]{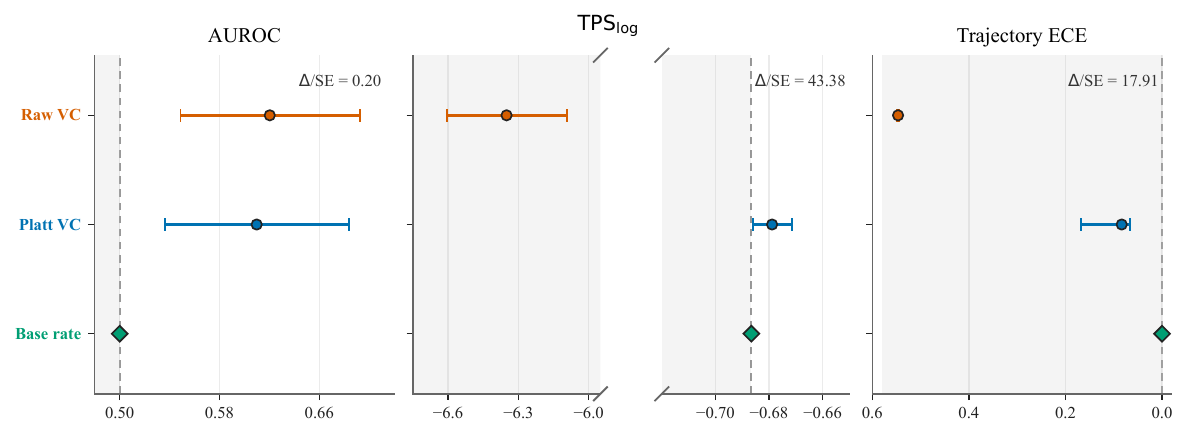}
  \caption{%
    Tau2-Bench calibration gap (\(n=201\)). Raw verbal confidence
    (orange), Platt-recalibrated confidence (blue), and the base-rate reference
    (green). Whiskers are 95\% bootstrap intervals; shaded regions are worse
    than base rate. Right is better; T-ECE is reversed and \(\TPSlog\) uses a
    broken axis.
  }
  \label{fig:exp2a}
\end{figure}
 
\subsection{Controlled censoring validation}
\label{sec:exp3s1}
\label{sec:exp3-stage1}

Before applying the censored estimator to natural truncation, we validate it
under artificial censoring on complete StrategyQA, Tau2-Bench, and HotpotQA
trajectories. Appendix~\ref{app:robustness} shows that the closed-form
prefix-swap/tail-omission decomposition matches the directly computed
\(\TPScensimple-\mathrm{TPS}^{\mathrm{complete}}\) difference to numerical
precision across score families, weight schedules, censoring rates, and
datasets. The validation also predicts the sign regime used to interpret
WebShop: low-confidence censored prefixes receive mild failure-branch
corrections, while overconfident censored prefixes are penalized sharply.
 
\subsection{Natural censoring on WebShop}
\label{sec:exp3s2}
\label{sec:exp3-stage2}

We next move from artificial censoring to naturally truncated agent traces.
WebShop is the relevant case: many trajectories hit the benchmark step
budget before an outcome is observed. Of 500 trajectories, 192 parse-error
truncations are excluded as informative censoring. These parse failures
arise from the agent's own malformed output and plausibly correlate with
task success, so treating them as administrative censoring would violate the
non-informative censoring assumption. The working sample contains \(n=308\)
trajectories: 163 completed trajectories with observed outcomes
(\(\delta=1\)) and 145 administrative max-step stops (\(\delta=0\)), for a
47.08\% administrative-censoring rate.

The censored subset is structurally different from the completed subset.
Admin-censored trajectories are \(3\times\) longer on average than completed
ones (30.0 vs.\ 9.88 observed steps), consistent with harder tasks exhausting
the step budget. Yet verbal confidence does not materially separate the two
subsets (\(\Delta\bar F=-0.007\)): the predictor is nearly flat along a
difficulty axis visible in the stopping pattern. Complete-only evaluation
therefore discards a large, structurally harder subset of the benchmark.
Figure~\ref{fig:exp3s2} shows this selection diagnostic alongside the
primary score shift. We report
\[
\Delta_{\mathrm{practice}}
=
\widehat{\mathrm{TPS}}_{\mathrm{censored\text{-}ext}}
-
\widehat{\mathrm{TPS}}_{\mathrm{complete\text{-}only}},
\]
with positive values meaning that censored-aware evaluation gives a higher
average score than complete-only evaluation.

Our main reported result uses Platt-calibrated verbal confidence scored with the log
rule and linear-front weights. Complete-only evaluation uses only the 163
completed trajectories and gives
$\widehat{\mathrm{TPS}}_{\mathrm{complete\text{-}only}}=-0.816$ nats; the
simple censored extension also includes the 145 max-step prefixes and gives
$\widehat{\mathrm{TPS}}_{\mathrm{censored\text{-}ext}}=-0.657$ nats.
The paired shift is $+0.159$ nats, with a 95\% bootstrap CI of
$[+0.133,+0.188]$. This is an evaluator effect, not an
improvement in task performance: the additional trajectories have unobserved
outcomes and are scored under the failure-side \(q_z\approx0\)
approximation. Normalized by the 47.08\% censoring rate, the shift is
\(+0.337\) nats, close to the controlled artificial-censoring value on Tau2
reported in Section~\ref{sec:exp3s1} (\(+0.404\) nats).

Across the score-family and weight-schedule sweep, the sign pattern is stable:
all non-reference predictor rows are positive except completion token
probability, the only predictor with \(\bar F>0.5\).  Under log score with
linear-front weights, this predictor receives a \(-2.07\) nat shift. The pattern is consistent with the artificial-censoring analysis in
Section~\ref{sec:exp3s1}: low-confidence censored prefixes receive a mild
failure-branch correction, while overconfident censored prefixes are
penalized sharply. Thus censored-aware scoring changes not only the average
score, but also how different confidence regimes are treated.
 
\begin{figure}[t]
  \centering
  \includegraphics[width=\linewidth]{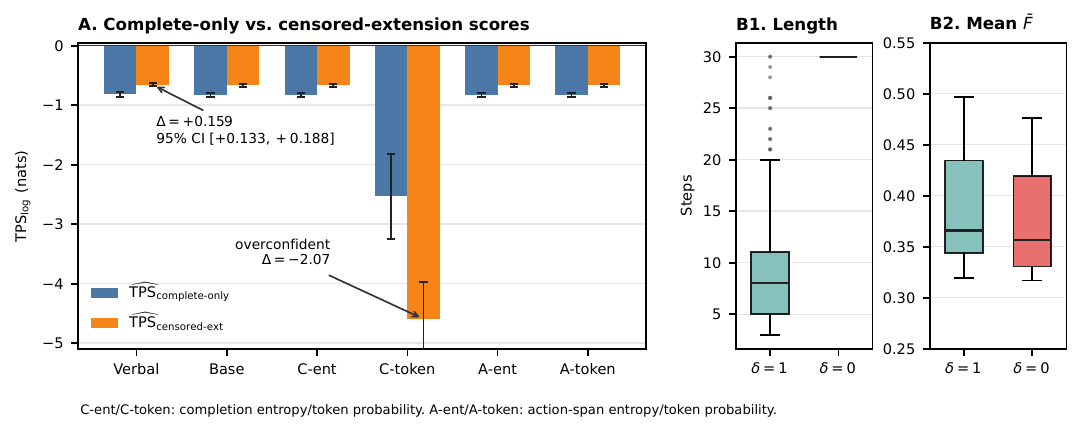}
  \caption{%
    WebShop natural censoring. Panel A compares complete-only and
    simple-censored scores; the primary verbal-confidence score shifts by
    \(+0.159\) nats with 95\% CI \([+0.133,+0.188]\). Panel B shows the
    selection diagnostic: admin-censored trajectories are longer, while verbal
    confidence barely separates completed and censored subsets. Whiskers are
    95\% bootstrap intervals.
  }
  \label{fig:exp3s2}
\end{figure}
  
\section{Conclusion and Limitations}
\label{sec:conclusion}

Evaluating agentic uncertainty requires scoring a probability trace, not only
summarizing it. Existing trajectory evaluators, including AUROC, T-ECE, and
scalarized T-Brier under common aggregators, elicit rank orderings, binwise
calibration, or collapsed scalars. None strictly elicits the
prefix-conditioned success process \((q_t)\) that a calibrated agent should
report. \(\TPS\), built from any strictly proper binary score and positive
trajectory weights, provides a trajectory-level evaluator that does.

Censoring further changes the evaluation population. When benchmarks impose
step budgets, complete-only evaluation silently discards truncated
trajectories. On WebShop, the administratively censored subset is structurally
different from the completed subset, and censored-aware scoring changes the
verdict by \(0.159\) nats with a 95\% confidence interval excluding zero.
Together with the Tau2-Bench recalibration result, where \(\TPSlog\) moves by
\(\Delta/\mathrm{SE}\approx43\) while AUROC moves by only
\(\Delta/\mathrm{SE}\approx0.3\), these results show that evaluator choice is
not cosmetic when \(F_t\) is consumed as a probability for deferral,
thresholding, reflection, or human handoff.

The framework has two main limitations. First, the censored extension requires
non-informative administrative censoring
(Assumptions~\ref{asm:noninformative}--\ref{asm:admin}); adaptive
monitor-driven stops and parse-error truncations violate these assumptions and
require informative-censoring methods such as IPCW or doubly robust
extensions. Second, the present construction targets binary terminal outcomes;
partial credit, graded success, multiple objectives, and vector-valued rewards
require extensions beyond the binary proper-score framework developed here.

Empirically, our trajectories use one model and one ReAct-style harness.
Broader validation across model families and harnesses is useful future work
for predictor-side coverage, but is separate from the evaluator-side propriety
claims established here.

\bibliographystyle{unsrtnat}
\bibliography{references}


\appendix

\section*{Appendix}

\section{Proof of Theorem~\ref{thm:complete} (Complete Observation)}
\label{app:proof_thm1}
\label{app:complete-proof}
 
We give the full proof of Theorem~\ref{thm:complete} and record three
remarks clarifying the filtration argument, the prequential interpretation,
and the role of the underlying binary scoring-rule family.

\begin{proof}
Let $F_{1:T}$ be any adapted forecast sequence with $F_t$
$\Hcal_t$-measurable, and let
\[
q_t := \Pp(Y{=}1\mid\Hcal_t).
\]
The trajectory weights $w_{1:T}$ are fixed by the evaluator and satisfy
$w_t>0$ for every $t$. By linearity of expectation,
\begin{equation}
  \Ep[\TPS(F_{1:T},Y)]
  =
  \sum_{t=1}^{T} w_t\,\Ep[S_{\alpha,\beta}(F_t,Y)].
  \label{eq:app_linear}
\end{equation}

Fix a step $t$. Since $F_t$ is $\Hcal_t$-measurable and
$Y\mid \Hcal_t$ has Bernoulli parameter $q_t$, conditional strict
propriety of the binary score gives
\begin{equation}
  R_t
  :=
  \Ep\!\left[
    S_{\alpha,\beta}(q_t,Y)
    -
    S_{\alpha,\beta}(F_t,Y)
    \mid \Hcal_t
  \right]
  \ge 0
  \quad \text{a.s.},
  \label{eq:app_Rt}
\end{equation}
with equality if and only if $F_t=q_t$ a.s.

Multiplying~\eqref{eq:app_Rt} by $w_t>0$, taking expectations, and
summing over $t$ yields
\begin{equation}
  \sum_{t=1}^{T} w_t\,\Ep[R_t]
  =
  \Ep[\TPS(q_{1:T},Y)]
  -
  \Ep[\TPS(F_{1:T},Y)]
  \ge 0.
  \label{eq:app_sum}
\end{equation}
Thus the truthful process $q_{1:T}$ maximizes expected trajectory score.

It remains to show uniqueness. If equality holds in~\eqref{eq:app_sum},
then each term in the nonnegative weighted sum must be zero:
$\Ep[R_t]=0$ for every $t$. Since $R_t\ge0$ a.s., this implies
$R_t=0$ a.s. for every $t$. By the equality condition in the per-step
strict propriety statement, $F_t=q_t$ a.s. for every $t$. Hence
$q_{1:T}$ is the unique maximizer of expected trajectory score, up to
almost-sure equivalence.
\end{proof}

\begin{remark}[Step dependence does not affect the proof]
\label{rem:step_dep}
The forecasts $(F_1,\ldots,F_T)$ may be arbitrarily dependent across
steps, and every step may target the same terminal outcome $Y$. The proof
does not require independent step labels or independent forecast errors.
The key point is conditional: at each filtration level $\Hcal_t$, the
reported forecast $F_t$ is measurable with respect to the information
available at that step, and the conditional law of $Y$ is Bernoulli with
parameter $q_t$. Strict propriety is therefore applied separately inside
each conditional expectation, and the resulting nonnegative regrets are
then aggregated with positive weights.
\end{remark}
 
\begin{remark}[Prequential interpretation]
\label{rem:prequential}
The construction is prequential in spirit: the evaluator scores a sequence
of probability forecasts as information accumulates along the interaction
history. It differs from classical one-step-ahead prequential forecasting
because all forecasts target the same terminal binary outcome $Y$, rather
than a sequence of distinct future observations. Accordingly,
\[
  q_t=\Ep[Y\mid\Hcal_t]
\]
is the truthful success probability given the information available at
step $t$. As the history grows, this conditional probability can increase
or decrease depending on what the agent observes or does. This is why no
monotonicity assumption is needed. A truthful success forecast may rise
after useful evidence or successful tool use, and may fall after a harmful
action or failed observation.
\end{remark}

\begin{remark}[Role of the binary scoring-rule family]
\label{rem:binary_family_appA}
Theorem~\ref{thm:complete} assumes that the per-step binary score
$S_{\alpha,\beta}$ is strictly proper. That property is established
for the beta family in Appendix~\ref{app:beta}. The proof above uses only
the consequence of that result: for any $\Hcal_t$-measurable forecast
$F_t$, the conditional expected score is uniquely maximized by reporting
$q_t=\Pp(Y{=}1\mid\Hcal_t)$.

Gneiting--Raftery gives the convex-function characterization of proper
scores; the binary specialization yields the Savage form; and the
Schervish threshold-mixture representation gives a convenient route to
strict propriety for beta-family scores whose mixing measure has full
support on $(0,1)$. The logarithmic endpoint is handled as the usual
strictly proper binary log score.
\end{remark}

\begin{remark}[Scope of the fixed-weight statement]
\label{rem:fixed_weight_scope}
The theorem is stated for a fixed evaluated horizon and fixed positive
evaluator weights. Length-specific reporting conventions, such as using a
normalized linear-front schedule separately for each observed complete
trajectory length, should be read as applying the same theorem within each
fixed length convention. Stopped or censored prefixes are handled by the
separate conditional-projection construction in Appendix~\ref{app:censored-proof}.
\end{remark}

\section{Proofs of Theorems A and B}
\label{app:proof_neg}

The results in this appendix should not be read as dismissing rank metrics,
ECE, or scalarized Brier scores. AUROC, AUPRC, and risk--coverage evaluate
ranking; ECE diagnoses binwise marginal calibration; and proper binary losses
such as Brier and log loss are proper for a single scalar forecast. The
negative results below show only that these evaluator families do not strictly
elicit the full prefix-conditioned forecast process $(q_t)$.

\subsection{Theorem A: T-ECE is Resolution-blind}

Let \(C=\Phi(F_{1:T})\in[0,1]\) be a scalar trajectory confidence obtained
from a forecast trace by an aggregator \(\Phi\), and let
\(\mathcal B=\{B_1,\ldots,B_K\}\) be a fixed partition of \([0,1]\). The
population trajectory ECE is
\begin{equation}
  \mathrm{T\text{-}ECE}_{\mathcal B}(C)
  =
  \sum_{k=1}^{K}
  \Pp(C\in B_k)
  \left|
    \Ep[Y\mid C\in B_k]
    -
    \Ep[C\mid C\in B_k]
  \right|.
  \label{eq:app_tece}
\end{equation}
This is a calibration functional, not a single-observation scoring rule.

The Reliability--Resolution--Uncertainty decomposition
\cite{Murphy1973,Degroot1983,Brocker2009} states that, for a negatively
oriented proper loss \(L\) and scalar forecast \(C\),
\begin{equation}
  \Ep[L(C,Y)]
  =
  \underbrace{
    \Ep\,d\!\left(C,\Ep[Y\mid C]\right)
  }_{\text{Reliability}}
  -
  \underbrace{
    \Ep\,d\!\left(\Ep[Y\mid C],\bar Y\right)
  }_{\text{Resolution}}
  +
  \underbrace{
    d(\bar Y,Y)
  }_{\text{Uncertainty}},
  \label{eq:app_rru}
\end{equation}
where \(d\) is the associated Bregman divergence and
\(\bar Y=\Ep[Y]\). T-ECE evaluates the Reliability component only: it asks
whether empirical success frequencies match average confidence within bins.
It does not penalize loss of Resolution, the component that distinguishes an
informative truthful forecast from an uninformative but calibrated one.

\begin{theoremA}[T-ECE is resolution-blind]
\labeltheoremA{thm:A}
There exists a data-generating distribution and a non-truthful forecast
process \(G\neq q\) almost surely such that
\[
  \mathrm{T\text{-}ECE}_{\mathcal B}(\Phi(G_{1:T}))
  \le
  \mathrm{T\text{-}ECE}_{\mathcal B}(\Phi(q_{1:T})).
\]
Moreover, binwise marginalization of the truthful forecast can preserve zero
binwise calibration error while strictly reducing Resolution whenever the
truthful conditional probability varies within bins.
\end{theoremA}

\begin{proof}
It suffices to take \(T=1\). Let \(H\in\{a,b\}\), with
\(\Pp(H=a)=\Pp(H=b)=\tfrac12\), and let
\[
  \Pp(Y=1\mid H=a)=0.2,
  \qquad
  \Pp(Y=1\mid H=b)=0.8.
\]
The truthful forecast is \(q(H)=\Pp(Y=1\mid H)\). Since
\(\Ep[Y\mid q=p]=p\), the truthful forecast achieves
\(\mathrm{T\text{-}ECE}(q)=0\) for any binning that separates or contains
these forecast values without inducing binwise miscalibration.

Now define the constant forecast \(G(H)\equiv 0.5\). Since the marginal
success rate is
\[
  \Ep[Y]
  =
  \tfrac12(0.2)+\tfrac12(0.8)
  =
  0.5,
\]
we have \(\Ep[Y\mid G=0.5]=0.5\), and hence
\(\mathrm{T\text{-}ECE}(G)=0\). Thus \(G\) ties the truthful forecast under
T-ECE despite \(G\neq q\) almost surely.

The difference is Resolution. Under the Brier divergence, the constant
forecast has zero Resolution, while the truthful forecast achieves
\[
  \tfrac12(0.2-0.5)^2+\tfrac12(0.8-0.5)^2
  =
  0.09
  >
  0.
\]
T-ECE is blind to this distinction.

The same phenomenon holds more generally for forecasts that replace \(q\)
by its conditional average on the cells of a fixed bin map. Such forecasts
preserve binwise marginal calibration by construction, but remove within-bin
variation. Whenever \(q\) has positive within-bin variance, this averaging
strictly reduces the Resolution term in~\eqref{eq:app_rru}, while the
T-ECE functional in~\eqref{eq:app_tece} does not penalize that loss.
Therefore T-ECE is not strictly proper for the full prefix-conditioned
forecast process.
\end{proof}

\subsection{Theorem B: Scalarized proper losses are not strictly proper for the trace}

Let \(L(p,Y)\) be a negatively oriented strictly proper binary loss, so lower
expected loss is better. Given an aggregator \(\Phi\), define the scalarized
trajectory loss
\begin{equation}
  \mathcal L_{\Phi}(F_{1:T},Y)
  :=
  L\!\left(\Phi(F_{1:T}),Y\right).
  \label{eq:app_scalarized_loss}
\end{equation}
This includes scalarized trajectory Brier as the special case
\(L(p,Y)=(p-Y)^2\). The issue is not that Brier or log loss are improper for
a single binary probability; they are proper for that scalar. The issue is
that applying them after a deterministic collapse \(\Phi(F_{1:T})\) changes
the elicited object from the trace \((F_t)\) to the scalar \(C=\Phi(F_{1:T})\).

Write
\[
  q_T^* := \Ep[Y\mid\Hcal_T]
\]
for the conditionally optimal scalar input to a strictly proper binary loss
given the final observed history. In the two-step examples below,
\(q_T^*=q_2\).

\begin{theoremB}[Scalarized proper losses elicit the collapsed scalar, not the trace]
\labeltheoremB{thm:B}
Let \(L\) be a strictly proper binary loss and let \(\Phi\) be one of the
AUQ-style aggregators \(\Phi_{\mathrm{last}}\), \(\Phi_{\mathrm{avg}}\),
\(\Phi_{\min}\), or a linear aggregator
\[
  \Phi_w(F_{1:T})=\sum_{t=1}^{T}w_tF_t,
  \qquad
  w_t\ge0,
  \qquad
  \sum_{t=1}^{T}w_t=1,
\]
with at least two positive weights.

\begin{enumerate}
  \item[(i)] \emph{Non-uniqueness.} For each of these aggregators, there
  exists a data-generating distribution and an adapted non-truthful trace
  \(G_{1:T}\neq q_{1:T}\) such that
  \[
    \Ep[\mathcal L_{\Phi}(G_{1:T},Y)]
    =
    \Ep[\mathcal L_{\Phi}(q_{1:T},Y)].
  \]
  Hence the scalarized loss is not strictly proper for the full trace.

  \item[(ii)] \emph{Strict suboptimality.} For
  \(\Phi_{\mathrm{avg}}\), \(\Phi_{\min}\), and any \(\Phi_w\) with
  \(T\ge2\) and at least two positive weights, there are data-generating
  distributions for which the truthful aggregate \(\Phi(q_{1:T})\) differs
  from the conditionally optimal scalar \(q_T^*\) on a positive-probability
  event, and an adapted non-truthful trace \(G_{1:T}\) can achieve that
  scalar optimum. In such cases,
  \[
    \Ep[\mathcal L_{\Phi}(G_{1:T},Y)]
    <
    \Ep[\mathcal L_{\Phi}(q_{1:T},Y)].
  \]
\end{enumerate}
For \(\Phi_{\mathrm{last}}\), part~(i) holds but part~(ii) does not: the
truthful trace is optimal, but not uniquely optimal, because the prefix is
invisible to the loss.
\end{theoremB}

We prove the aggregator-specific cases below.

\subsubsection*{Corollary B1: \(\Phi_{\mathrm{last}}\) is proper for the last scalar, not strictly proper for the trace}

Let
\[
  \Phi_{\mathrm{last}}(F_{1:T})=F_T.
\]
Then \(\mathcal L_{\Phi_{\mathrm{last}}}(F,Y)=L(F_T,Y)\) depends only on
the final forecast. By strict propriety of \(L\), the expected loss is
minimized by \(F_T=q_T\). However, any adapted trace \(G_{1:T}\) satisfying
\(G_T=q_T\) and \(G_t\neq q_t\) on a positive-probability set for some
\(t<T\) achieves the same expected loss as the truthful trace. Thus
\(\Phi_{\mathrm{last}}\)-Brier, and likewise any last-step scalarized proper
loss, is proper for the collapsed final scalar but not strictly proper for
the full trace.

\subsubsection*{Corollary B2: \(\Phi_{\mathrm{avg}}\) can make the truthful trace strictly suboptimal}
 
Let \(T=2\). Let \(q_1=0.5\), and let
\[
  q_2\in\{0.7,0.3\}
\]
with each value occurring with probability \(1/2\), where
\(\Ep[q_2\mid\Hcal_1]=q_1=0.5\) by the law of iterated expectations. This is
a valid prefix-conditioned success process.

For the truthful trace,
\[
  \Phi_{\mathrm{avg}}(q)
  =
  \frac{q_1+q_2}{2}
  \in
  \{0.6,0.4\}.
\]
But, conditional on \(\Hcal_2\), the strictly proper scalar loss is minimized
by reporting \(q_T^*=q_2\), not \((q_1+q_2)/2\).

Define an adapted non-truthful trace
\[
  G_1=0.5,
  \qquad
  G_2=2q_2-0.5
  \in
  \{0.9,0.1\}.
\]
Then
\[
  \Phi_{\mathrm{avg}}(G)
  =
  \frac{G_1+G_2}{2}
  =
  q_2.
\]
Thus \(G\) achieves the conditionally optimal scalar \(q_T^*=q_2\), while
the truthful trace does not. For the Brier loss, the truthful aggregate is
off by \(0.1\) on every realization, so it incurs an excess conditional loss of
\(0.1^2=0.01\). Hence the truthful trace is strictly suboptimal under
\(\Phi_{\mathrm{avg}}\)-Brier.

\subsubsection*{Corollary B3: \(\Phi_{\min}\) can make the truthful trace strictly suboptimal}

Use the same two-step process as above:
\[
  q_1=0.5,
  \qquad
  q_2\in\{0.7,0.3\}.
\]
For the truthful trace,
\[
  \Phi_{\min}(q)=\min(q_1,q_2)\in\{0.5,0.3\}.
\]
Define an adapted non-truthful trace
\[
  G_1=1,
  \qquad
  G_2=q_2.
\]
Then
\[
  \Phi_{\min}(G)=q_2.
\]
On the event \(q_2=0.7\), the truthful scalar is \(0.5\), while the
non-truthful scalar is the conditionally optimal value \(q_T^*=q_2=0.7\).
Computing the non-truthful minus truthful Brier loss on this event gives
\[
  (0.7-0.7)^2-(0.5-0.7)^2
  =
  -0.04.
\]
On the event \(q_2=0.3\), both aggregates equal \(0.3\). Therefore the
non-truthful trace has strictly smaller expected Brier loss. The
non-minimum step \(G_1=1\) is invisible to the scalarized score, so the
aggregator can reward inflation of non-minimum prefixes.

\subsection{Weighted-average aggregators}
\label{sec:weighted_avg}

Let
\[
  \Phi_w(F_{1:T})=\sum_{t=1}^{T} w_t F_t,
  \qquad
  w_t\ge0,
  \qquad
  \sum_{t=1}^{T}w_t=1,
\]
with at least two positive weights. Choose indices \(i\neq j\) such that
\(w_i,w_j>0\), and choose \(c\neq0\) small enough that
\[
  G_i=q_i+c,
  \qquad
  G_j=q_j-\frac{w_i}{w_j}c
\]
remain in \((0,1)\) almost surely, with \(G_t=q_t\) for all
\(t\notin\{i,j\}\). Then
\[
  \Phi_w(G)=\Phi_w(q)
  \quad \text{a.s.}
\]
while \(G\neq q\). Thus any such weighted-average scalarization fails strict
propriety for the trace by non-uniqueness.

The strict-suboptimality clause follows under the achievability condition in
Theorem~\ref{thm:B}: whenever the truthful weighted aggregate
\(\Phi_w(q_{1:T})\) differs from the conditionally optimal scalar on a
positive-probability event, and an adapted non-truthful trace can make
\(\Phi_w(G_{1:T})\) equal that optimum, strict propriety of the scalar loss
implies
\[
  \Ep[\mathcal L_{\Phi_w}(G_{1:T},Y)]
  <
  \Ep[\mathcal L_{\Phi_w}(q_{1:T},Y)].
\]
The average-aggregator construction above is the special case
\(w_1=w_2=\tfrac12\). Therefore the averaging pathology is not specific to
the unweighted mean.

No deterministic scalarization used in existing agentic UQ work, including
last, average, minimum, or weighted average, strictly elicits the full
prefix-conditioned success-probability process. The issue is not that the
underlying scalar score is improper; the issue is that scalarization changes
the elicited object.
 
\section{Censored Proofs}
\label{app:censored}
\label{app:proof_cen}
\label{app:censored-proof}
 
This appendix proves the censored propriety results and records the
plug-in approximation criteria used to interpret the exact reduced score.

\begin{proof}[Proof of Theorem~\ref{thm:abstract_censored}
(Abstract censored propriety on the observable prefix)]
For a stopped prefix ending at \(Z\), the observable-prefix complete-data
\TPS is
\[
{\TPSobs}_{\alpha,\beta}(F_{1:Z},Y;Z)
=
\sum_{t=1}^{Z} w_t S_{\alpha,\beta}(F_t,Y).
\]
By complete-data propriety applied to the observable prefix, for any
adapted forecast process \(F_{1:Z}\),
\[
\Ep\!\left[
{\TPSobs}_{\alpha,\beta}(q_{1:Z},Y;Z)
\right]
\ge
\Ep\!\left[
{\TPSobs}_{\alpha,\beta}(F_{1:Z},Y;Z)
\right].
\]
By definition of the abstract censored score and the tower property,
\begin{align*}
\Ep\!\left[
{\TPScenabs}_{\alpha,\beta}(q_{1:Z};\Gcal_Z)
\right]
&=
\Ep\!\left[
  \Ep\!\left[
    {\TPSobs}_{\alpha,\beta}(q_{1:Z},Y;Z)
    \mid \Gcal_Z
  \right]
\right] \\
&=
\Ep\!\left[
{\TPSobs}_{\alpha,\beta}(q_{1:Z},Y;Z)
\right],
\\[4pt]
\Ep\!\left[
{\TPScenabs}_{\alpha,\beta}(F_{1:Z};\Gcal_Z)
\right]
&=
\Ep\!\left[
  \Ep\!\left[
    {\TPSobs}_{\alpha,\beta}(F_{1:Z},Y;Z)
    \mid \Gcal_Z
  \right]
\right] \\
&=
\Ep\!\left[
{\TPSobs}_{\alpha,\beta}(F_{1:Z},Y;Z)
\right].
\end{align*}
Substituting these identities into the complete-data inequality gives
\[
\Ep\!\left[
{\TPScenabs}_{\alpha,\beta}(q_{1:Z};\Gcal_Z)
\right]
\ge
\Ep\!\left[
{\TPScenabs}_{\alpha,\beta}(F_{1:Z};\Gcal_Z)
\right],
\]
which proves propriety of the abstract censored \TPS on the observable
prefix.
\end{proof}

\begin{remark}[Strictness is restricted to the observable prefix]
\label{rem:strictness_scope}
Censored observations contain no information about forecasts after the
stopping step. If two forecast processes agree on \(F_{1:Z}\) but differ on
the unobserved tail \(F_{Z+1:T}\), they induce the same conditional
projection onto \(\Gcal_Z\). Therefore no censored score can be strictly
proper for the full unobserved sequence \(F_{1:T}\) without additional
information about the tail. All strictness statements in the censored
setting are consequently restricted to the observable prefix \(F_{1:Z}\).
\end{remark}

We next record the branch-specific posterior weight that connects the
abstract projection to the exact reduced binary form.

\begin{proposition}[Branch-specific posterior identification]
\label{prop:identification}
On the censored branch \(\Delta=0\), define
\[
\eta_Z^{(0)}
:=
\Pp(Y=1\mid\Gcal_Z,\Delta=0).
\]
Under Assumption~\ref{asm:admin},
\[
\eta_Z^{(0)}
=
\Pp(Y=1\mid\Hcal_Z)
=
q_Z.
\]
Equivalently, the full posterior success weight given \(\Gcal_Z\) is
\[
\eta_Z^\star
:=
\Pp(Y=1\mid\Gcal_Z)
=
\Delta Y+(1-\Delta)q_Z .
\]
\end{proposition}

\begin{proof}
On complete trajectories, \(\Delta=1\), recall that
\(\Gcal_Z=\sigma(H_Z,Z,\Delta,\Delta Y)\). Thus the stopped-data
sigma-field contains the realized outcome through \(\Delta Y=Y\). Hence
\[
\Pp(Y=1\mid\Gcal_Z)=Y
\qquad \text{on } \{\Delta=1\}.
\]
On censored trajectories, \(\Delta=0\), the outcome is unobserved. By
Assumption~\ref{asm:admin}, the administrative stopping information adds no
outcome information beyond the stopped history:
\[
\Pp(Y=1\mid\Gcal_Z,\Delta=0)
=
\Pp(Y=1\mid\Hcal_Z)
=
q_Z,
\]
where the last equality is the definition of the continuation-success
target. Combining the complete and censored branches gives
\(\eta_Z^\star=\Delta Y+(1-\Delta)q_Z\).
\end{proof}

Since \(Y\in\{0,1\}\), the abstract censored score admits the binary
decomposition
\begin{equation}
{\TPScenabs}_{\alpha,\beta}(F_{1:Z};\Gcal_Z)
=
\eta_Z^\star\,
{\TPSobs}_{\alpha,\beta}(F_{1:Z},1;Z)
+
(1-\eta_Z^\star)\,
{\TPSobs}_{\alpha,\beta}(F_{1:Z},0;Z).
\label{eq:oracle_decomp}
\end{equation}
Using Proposition~\ref{prop:identification}, this decomposition is exactly
the complete branch plus the \(q_Z\)-weighted censored branch in
\eqref{eq:cen_exact}.

\begin{proof}[Proof of Theorem~\ref{thm:cen_exact}
(Exact reduced censored beta score)]
We show that \(\TPScenexact_{\alpha,\beta}\) is the explicit reduced form
of the abstract conditional projection.

On the complete branch, \(\Delta=1\), the realized outcome \(Y\) is
observed. Therefore the conditional projection of the observable-prefix
complete-data score is simply
\[
\sum_{t=1}^{Z} w_t S_{\alpha,\beta}(F_t,Y).
\]

On the censored branch, \(\Delta=0\), the outcome is unobserved. By
\eqref{eq:oracle_decomp} and Proposition~\ref{prop:identification}, the
posterior success weight on the censored branch is
\(\eta_Z^{(0)}=q_Z\). Thus the conditional projection is
\[
\sum_{t=1}^{Z} w_t
\left[
q_Z S_{\alpha,\beta}(F_t,1)
+
(1-q_Z)S_{\alpha,\beta}(F_t,0)
\right].
\]
Combining the complete and censored branches gives exactly
\[
  \mathrm{TPS}^{\mathrm{cen,exact}}_{\alpha,\beta}(F_{1:Z};\Gcal_Z)
  =
  \sum_{t=1}^{Z} w_t
  \Bigl[
    \Delta S_{\alpha,\beta}(F_t,Y)
    +
    (1-\Delta)
    \bigl(
      q_Z S_{\alpha,\beta}(F_t,1)
      +
      (1-q_Z)S_{\alpha,\beta}(F_t,0)
    \bigr)
  \Bigr],
\]
which is \eqref{eq:cen_exact}.

At the logarithmic endpoint, the exact reduced score is
\begin{equation}
  \begin{aligned}
  \mathrm{TPS}^{\mathrm{cen,exact}}_{\log}(F_{1:Z};\Gcal_Z)
  =
  \sum_{t=1}^{Z} w_t
  \Bigl[
    &\Delta\{Y\log F_t+(1-Y)\log(1-F_t)\} \\
    &+(1-\Delta)
      \{q_Z\log F_t+(1-q_Z)\log(1-F_t)\}
  \Bigr].
  \end{aligned}
  \label{eq:cen_exact_log}
\end{equation}

Propriety follows because this score is algebraically the abstract censored
projection from Theorem~\ref{thm:abstract_censored}. Hence its expected
score is maximized by the truthful process \(q_{1:Z}\).

If \(S_{\alpha,\beta}\) is strictly proper and \(w_t>0\) for all observed
steps, strictness holds on the observable prefix. By the tower-property
identities above, the expected regret of \(\TPScenexact_{\alpha,\beta}\)
equals the expected observable-prefix complete-data regret. Although \(Y\)
is unobserved on censored branches, it is well defined under the full-data
law. Thus strictness can be checked through the complete-data regret.

For each step, define the conditional proper regret
\[
R_t(F_t):=
q_t\{S_{\alpha,\beta}(q_t,1)-S_{\alpha,\beta}(F_t,1)\}
+
(1-q_t)\{S_{\alpha,\beta}(q_t,0)-S_{\alpha,\beta}(F_t,0)\}.
\]
By strict propriety of \(S_{\alpha,\beta}\), \(R_t(F_t)\ge 0\), with
equality if and only if \(F_t=q_t\) almost surely. The expected
observable-prefix complete-data regret is
\[
\Ep\!\left[
\sum_{t=1}^{T} \mathbf 1\{t\le Z\} w_t R_t(F_t)
\right].
\]
Every term in this sum is nonnegative, and \(w_t>0\) on observed steps. If
the total regret is zero, then \(R_t(F_t)=0\) almost surely on
\(\{t\le Z\}\). Therefore
\[
F_t=q_t
\qquad\text{a.s. on } \{t\le Z\}.
\]
Thus \(\TPScenexact_{\alpha,\beta}\) is strictly proper for the observable
prefix, but not for the unobserved tail.
\end{proof}
 
\begin{proposition}[Target of the simple censored score]
\label{prop:cen_simple_target}
For any strictly proper binary score \(S_{\alpha,\beta}\), the simple
censored score in~\eqref{eq:cen_simple} is strictly proper for the
observable pseudo-label target
\[
  m_t := \Ep[\Delta Y\mid\Hcal_t]
       = \Pp(\Delta=1,Y=1\mid\Hcal_t),
\]
on the observed prefix. It is strictly proper for the original
continuation-success target \(q_t=\Pp(Y=1\mid\Hcal_t)\) if and only if
\[
  \Pp(\Delta=0,Y=1\mid\Hcal_t)=0
\]
almost surely for the observed steps; equivalently, whenever
\(\Pp(Y=1\mid\Hcal_t)>0\),
\(\Pp(\Delta=0\mid\Hcal_t,Y=1)=0\).
\end{proposition}
 
\begin{proof}[Proof of Proposition~\ref{prop:cen_simple_target}]
The simple censored score replaces the unobserved outcome by the
pseudo-label
\[
\widetilde Y^{\mathrm{simple}}:=\Delta Y.
\]
For the logarithmic endpoint, the per-step simple score is
\[
S_t^{\mathrm{simple}}
=
\Delta Y\log F_t+(1-\Delta Y)\log(1-F_t).
\]
Conditioning on the available history gives
\[
\Ep[
S_t^{\mathrm{simple}}
\mid \Hcal_t
]
=
m_t\log F_t+(1-m_t)\log(1-F_t),
\qquad
m_t:=\Ep[\Delta Y\mid\Hcal_t].
\]
This conditional expected log score is uniquely maximized at
\[
F_t=m_t.
\]

The same argument applies to any strictly proper binary score
\(S_{\alpha,\beta}\). Conditional on \(\Hcal_t\), the pseudo-label
\(\Delta Y\) has Bernoulli mean \(m_t\). A strictly proper binary score is
therefore uniquely optimized by reporting \(m_t\). Summing over observed
steps with positive weights preserves strict propriety for the observable
pseudo-label target \(m_t\).

Finally,
\[
q_t-m_t
=
\Pp(Y=1\mid\Hcal_t)
-
\Pp(\Delta=1,Y=1\mid\Hcal_t)
=
\Pp(\Delta=0,Y=1\mid\Hcal_t).
\]
Thus the simple censored score elicits the original continuation-success
target \(q_t\) exactly if and only if
\[
\Pp(\Delta=0,Y=1\mid\Hcal_t)=0
\]
almost surely, equivalently
\[
\Pp(\Delta=0\mid\Hcal_t,Y=1)=0
\]
on histories where \(\Pp(Y=1\mid\Hcal_t)>0\). This is precisely the
zero missing-success-mass regime represented by the failure-side
\(q_Z\approx0\) approximation on censored prefixes.
\end{proof}

\begin{remark}[Plug-in propriety criterion]
\label{rem:plugin}
Suppose the exact censored score is implemented with an estimated
continuation-success weight \(\widehat q_Z\) on censored prefixes. Define
the soft surrogate outcome
\[
\widetilde Y_Z
:=
\Delta Y+(1-\Delta)\widehat q_Z.
\]
At an observed step \(t\le Z\), the plug-in per-step score has conditional
mean
\[
\Ep\!\left[
\widetilde Y_Z S_{\alpha,\beta}(p,1)
+
(1-\widetilde Y_Z)S_{\alpha,\beta}(p,0)
\mid \Hcal_t
\right].
\]
Let
\[
a_t:=\Ep[\widetilde Y_Z\mid\Hcal_t].
\]
By binary strict propriety, the conditional expected plug-in score is
optimized at \(p=a_t\). Therefore the plug-in score preserves propriety for
the original target \(q_t\) at step \(t\) if and only if
\[
\Ep[\widetilde Y_Z\mid\Hcal_t]=q_t.
\]

In particular, if \(\widehat q_Z=q_Z\) on censored prefixes, then
\[
\widetilde Y_Z
=
\Delta Y+(1-\Delta)q_Z
=
\Ep[Y\mid\Gcal_Z].
\]
Since \(\Hcal_t\subseteq\Gcal_Z\) on observed steps \(t\le Z\), the tower
property gives
\[
\Ep[\widetilde Y_Z\mid\Hcal_t]
=
\Ep\!\left[
\Ep[Y\mid\Gcal_Z]
\mid \Hcal_t
\right]
=
\Ep[Y\mid\Hcal_t]
=
q_t.
\]
Thus exact \(q_Z\) estimation is a clean sufficient condition for preserved
propriety, while the general criterion is the conditional-mean equality
above.
\end{remark}

\begin{remark}[Regret decomposition under plug-in error]
\label{rem:regret}
Let
\[
r_t
:=
\Ep[\widetilde Y_Z\mid\Hcal_t]-q_t
\]
be the conditional-mean bias of the plug-in surrogate at an observed step.
For any \(p\in[0,1]\), define the binary score contrast
\[
D(p):=S_{\alpha,\beta}(p,1)-S_{\alpha,\beta}(p,0).
\]
The per-step expected plug-in regret relative to reporting \(p\) decomposes
as
\[
\begin{aligned}
&\Ep\!\left[
\widetilde Y_Z
\{S_{\alpha,\beta}(q_t,1)-S_{\alpha,\beta}(p,1)\}
+
(1-\widetilde Y_Z)
\{S_{\alpha,\beta}(q_t,0)-S_{\alpha,\beta}(p,0)\}
\mid \Hcal_t
\right]
\\
&\qquad
=
R_t^\star(p)
+
r_t\bigl(D(q_t)-D(p)\bigr),
\end{aligned}
\]
where
\[
R_t^\star(p)
:=
q_t\{S_{\alpha,\beta}(q_t,1)-S_{\alpha,\beta}(p,1)\}
+
(1-q_t)\{S_{\alpha,\beta}(q_t,0)-S_{\alpha,\beta}(p,0)\}
\ge 0
\]
is the oracle proper regret. For interior bounded beta-family members
\((\alpha,\beta>0)\), the contrast \(D\) is bounded by some constant
\(B_{\alpha,\beta}<\infty\), so
\[
\left|
r_t\bigl(D(q_t)-D(p)\bigr)
\right|
\le
2B_{\alpha,\beta}|r_t|.
\]
After summing with weights \(w_t\), this gives an approximate-propriety
bound over the observable prefix: the plug-in score differs from the oracle
proper regret by a perturbation proportional to the conditional-mean bias
\(|r_t|\). The logarithmic endpoint is unbounded, so this bounded-error
statement applies only to interior beta-family members.
\end{remark}

\section{Calibration Procedure}
\label{app:calibration}

We recalibrate per-step forecast streams using single-split cross-fitted
Platt scaling. The procedure is applied consistently across the experiments
whenever a Platt-calibrated predictor variant is reported. This appendix
documents the split construction, logit-space feature transform, weighted
logistic fit, monotonicity safeguard, and cross-fitting procedure.
 
\noindent\textbf{Cross-fitting design.}
We use a single \(50/50\) trajectory-level split rather than \(K\)-fold
cross-fitting so that each calibrated stream is produced by only two
held-out calibration maps. With \(K\) independent fits on weak signals,
fold-specific slopes can be unstable or sign-inconsistent; concatenating
those maps can introduce non-monotone calibrated outputs.

\noindent\textbf{Splitting.}
Trajectories are partitioned into two equally sized halves stratified by
outcome label \(Y\). Within each label, trajectories are ordered by a
stable trajectory identifier and alternated between halves, yielding
approximately balanced success/failure distributions in both splits. All
per-step records inherit the split assignment of their parent trajectory, so no
within-trajectory information leaks across the calibration split.

\noindent\textbf{Feature construction.}
Each raw per-step forecast \(F_t\) is first clipped into
\([\varepsilon,1-\varepsilon]\) with \(\varepsilon=10^{-6}\) and mapped to
log-odds:
\[
x_{\mathrm{raw}}
=
\mathrm{logit}(F_t)
=
\log\!\left(\frac{F_t}{1-F_t}\right).
\]
Within each training split, a weighted mean \(\mu_{\mathrm{train}}\) and
weighted standard deviation \(\sigma_{\mathrm{train}}\) are computed from
\(x_{\mathrm{raw}}\) using the same linear-front step weights \(w_t\) used
by \TPS. The standardized feature is
\[
x
=
\frac{x_{\mathrm{raw}}-\mu_{\mathrm{train}}}
     {\max(\sigma_{\mathrm{train}},10^{-6})}.
\]
Held-out trajectories use the training split's
\(\mu_{\mathrm{train}}\) and \(\sigma_{\mathrm{train}}\); no statistics
from the held-out half enter the standardization.

\noindent\textbf{Weighted logistic fit.}
Within each training split we fit
\[
p_{\mathrm{cal}}
=
\mathrm{sigmoid}(a+b\,x)
\]
by weighted maximum likelihood on per-step records, where
\(\mathrm{sigmoid}(u)=(1+e^{-u})^{-1}\). The fit uses sample weights
\(w_t\) and an \(L_2\) penalty of \(\lambda=1.0\) on the slope \(b\)
only. Sample weights are not class-balanced; the calibrator targets the
empirical Bernoulli success labels rather than an artificial \(50/50\)
prior.

\noindent\textbf{Monotone non-decreasing fallback.}
If the fitted slope satisfies \(b<0\), we set \(b=0\) and
\[
a=\mathrm{logit}(\bar y_{\mathrm{train}}),
\]
where \(\bar y_{\mathrm{train}}\) is the weighted in-train success rate.
This enforces a non-decreasing mapping and prevents pathological sign
flips on signals with no usable resolution. The intercept-only fallback
returns the in-train base rate at every input. After computing
\(p_{\mathrm{cal}}\), a final \(\varepsilon\)-clip to
\([10^{-6},1-10^{-6}]\) is applied before scoring.

\noindent\textbf{Cross-fitting.}
The model fitted on split A is applied to split B, and the model fitted on
split B is applied to split A. Each trajectory is therefore calibrated by
a model that did not observe its parent split during training.

\noindent\textbf{Numerical stabilization.}
Verbal-confidence streams concentrate near \(\{0.90,0.95,1.00\}\), so
their log-odds span a range dominated by the \(\varepsilon\)-clip jump from
\(\mathrm{logit}(0.95)\approx2.94\) to
\(\mathrm{logit}(1-\varepsilon)\approx13.82\). Per-split
standardization prevents this clipped boundary from dominating the logistic
fit. The \(L_2\) penalty \(\lambda=1.0\) is used to shrink unstable slopes
on small calibration splits with weakly informative signals. Calibrating
with the same linear-front weights as \TPS aligns the calibrator with the
weighted objective used during evaluation.

\noindent\textbf{Diagnostics on Tau2-Bench verbal confidences.}
The procedure was applied to the \(n=201\) uncensored Tau2-Bench
trajectories used in the calibration-invariance experiment, with empirical success rate
\(\bar y=0.443\).

\begin{table}[h]
\centering
\small
\begin{tabular}{lc}
\toprule
Quantity & Value \\
\midrule
Split A slope \(b\) & \(0.2639\) (no fallback) \\
Split B slope \(b\) & \(0.2213\) (no fallback) \\
Fallback triggered & No \\
Raw values \(0.90,0.95,1.00\) map to & \(0.33,0.34,0.47\) \\
Calibrated probability range & \([0.33,0.47]\) \\
\bottomrule
\end{tabular}
\bottomtablecaption{Single-split Platt calibration diagnostics on Tau2-Bench
verbal confidences. Both splits return positive finite slopes,
so the monotone fallback is not triggered.}
\label{tab:platt_diag}
\end{table}
  
\noindent
The corresponding metric changes are summarized in
Section~\ref{sec:exp2a}; this table reports only calibration diagnostics.

\noindent\textbf{Limitations.}
We use single-split Platt scaling as a simple recalibration device, not as an
optimal calibrator. Alternative monotone or nonparametric calibrators such as
isotonic, beta calibration, or Venn--Abers may produce different calibrated
streams. The fixed-rank robustness check isolates the structural evaluator claim by
holding the rank-metric input fixed, so the rank-vs-\(\TPS\) gap does not
depend on Platt scaling.

The procedure is validated primarily on the Tau2 verbal-confidence stream.
Other predictors admit the same procedure mechanically, but their slopes
and calibrated ranges differ and are reported in
Appendix~\ref{app:transparency}. Weighted standard deviation can be small
on highly saturated streams; in that case the
\(\sigma_{\mathrm{train}}\ge10^{-6}\) floor activates and
standardization degenerates. This floor was not active in our runs, but it
is a known failure mode for streams with effectively a single unique value.
When the evaluator weight schedule changes, for example from linear-front
to uniform, the calibrator must be refit with matching weights; otherwise
the calibration-target mismatch is conflated with the intended
score-family or weighting comparison. Finally, the cross-fitted procedure
produces two distinct \((a,b)\) pairs for the same predictor, so
trajectories in different splits are calibrated under slightly different
mappings. This is a known property of held-out calibration rather than an
error: every reported calibrated trajectory is scored under a map fitted
without using that trajectory's outcome.

\section{Fixed-Rank Sensitivity Analysis}
\label{app:fixed_rank}

To check that the calibration-invariance gap in Section~\ref{sec:exp2a} is not
a Platt-scaling artifact, we repeat the comparison with the input to the
rank-based metrics held fixed. On the same Tau2 traces, we compute one
trajectory scalar \(r_i=\Phi(F^{\mathrm{raw}}_{i,1:T_i})\) from the raw stream
and reuse this same \(r_i\) for AUROC, AUPRC, and AURC in every row. Only the
per-step probability stream evaluated by \(\TPSlog\) is transformed.

We apply five transforms chosen to span qualitatively distinct probability-scale
distortions: identity, affine compression \(g(x)=0.4+0.2x\), square root,
square, and the empirical Platt mapping from Section~\ref{sec:exp2a}, covering
linear, sublinear, superlinear, and learned rescalings. AUROC, AUPRC, and AURC
are identical by construction across all five, while \(\TPSlog\) spans
\(5.7\) nats. The label-free affine compression provides the cleanest contrast:
it shifts the probability scale without using any outcome information, yet the
rank metrics remain fixed. The gap is structural to rank-based evaluation, not
a property of any particular calibrator.

\section{Predictor Transparency Tables}
\label{app:transparency}

This appendix reports complete-observation predictor transparency tables for
StrategyQA, Tau2-Bench, and HotpotQA. All rows in this appendix have an
observed terminal outcome \(Y\); administratively censored and informatively
terminated trajectories are excluded here and handled separately in
Appendices~\ref{app:robustness:artificial},
\ref{app:robustness:webshop}, and \ref{app:parse_error}. Thus, this appendix
is restricted to the uncensored evaluation setting of
\S\ref{sec:tps_complete}.

These tables are provided for transparency rather than for predictor selection.
They document how the reporting convention of \S\ref{sec:reporting} applies
across the full predictor pool and across three complete-observation datasets.
They also distinguish per-step probability streams, which are valid inputs to
\(\TPS\), from AUQ-style scalar aggregators, which collapse a trajectory before
evaluation and therefore fall outside the domain of \(\TPS\).

\noindent\textbf{Sample definition.}
The transparency tables are restricted to completed trajectories for which the
benchmark returned a terminal success/failure label. The resulting samples are
StrategyQA (\(n=2229\)), Tau2-Bench (\(n=201\)), and HotpotQA
(\(n=1529\)). Trajectories ending because of step-limit exhaustion, parser
failure, tool-interface failure, or environment termination without a clean
task label are excluded from this appendix and handled in the termination audit
and censored-analysis appendices. For predictors defined only on action spans,
the available sample may be slightly smaller; such cases are noted in the
corresponding table captions.

\noindent\textbf{Per-step predictor pool.}
For each dataset, we report five per-step confidence streams, all stored on
\([0,1]\) with larger values indicating higher predicted probability of
eventual task success. These are: verbal confidence; completion-token
probability; completion entropy confidence; action-span token probability;
and action-span entropy confidence. The completion-level streams are computed
over the full generated step, while the action-span streams restrict the
signal to the parsed action that the agent commits to executing. We also
include a constant empirical base-rate stream, \(F_t\equiv \bar y\), as an
uninformative reference. Each non-constant stream is shown in raw form and,
when available, after single-split cross-fitted Platt recalibration. The
recalibration procedure is described in Appendix~\ref{app:calibration}.

\emph{Verbal confidence} is the scalar probability of success that the model
self-reports at each step, parsed from the structured confidence field and
clipped to \([0,1]\); missing or malformed confidence fields yield a missing
per-step value. \emph{Completion-token probability} is the mean chosen-token
probability across all tokens generated at that step, capturing how peaked the
next-token distribution was on average over the model's entire
reasoning-plus-action output. \emph{Completion-entropy confidence} measures the
same step's distributional uncertainty as the average per-token Shannon entropy
over the top-\(5\) candidates, in nats, and inverts it to a confidence via
\(1-H/\ln 5\). \emph{Action-span token probability} and \emph{action-span
entropy confidence} are the same two signals restricted to the tokens in the
parsed committed action: the completion variants reflect confidence over the
entire deliberation, while the action-span variants isolate confidence in the
action the agent actually executes.

\noindent\textbf{Metrics.}
All metric columns carry an explicit orientation arrow. AUROC and AUPRC are
failure-oriented, with positive class \(Y=0\), and are computed from a scalar
trajectory risk summary. AURC is reported as a risk--coverage diagnostic, with
lower values better. Trajectory ECE is reported as tie-aware quantile-binned
calibration error with 10 bins. Scalarized trajectory Brier is the AUQ-style
trajectory reliability diagnostic computed on the same scalar trajectory
summary. The \(\TPS\) columns are reward-oriented, so higher values are better,
and score the full per-step probability stream using normalized linear-front
weights.

A small AUROC change between raw and Platt rows should not be read as a
violation of AUROC's monotone-invariance property. That invariance applies to a
single global monotone transformation. Our Platt rows are cross-fitted: separate
fold-specific calibration maps are fit and then pooled, so inter-fold scale
differences can perturb pooled trajectory rankings slightly. The fixed-rank
robustness check in Appendix~\ref{app:fixed_rank} isolates the theoretical invariance case by
holding the AUROC/AUPRC/AURC input scalar fixed by construction.

\noindent\textbf{Scalar AUQ-style aggregators.}
AUQ-style scalar aggregators such as
\(\Phi_{\mathrm{last}}\), \(\Phi_{\mathrm{avg}}\), and \(\Phi_{\min}\)
produce one scalar per trajectory rather than an adapted per-step probability
stream. They are therefore not scored by \(\TPS\). We report them in a
separate scalar-predictor table using only scalar-compatible diagnostics:
AUROC, AUPRC, AURC, trajectory ECE, and scalarized trajectory Brier. The
\(\TPS\) entries are omitted because these scalar predictors do not claim to
provide the full prefix-conditioned process \((F_t)_{t=1}^T\).

\subsection{Per-step predictor transparency}
\label{app:transparency_perstep}

Tables~\ref{tab:transparency_strategyqa}--\ref{tab:transparency_hotpotqa}
report complete-observation metrics for per-step confidence streams. These
predictors fall within the domain of \(\TPS\), because they provide a
probability-like value at each trajectory prefix. For all three tables, AUROC,
AUPRC, and AURC are failure-oriented scalar diagnostics; T-ECE is tie-aware
quantile-binned calibration error with 10 bins; T-Brier is a scalar trajectory
diagnostic; and the \(\TPS\) columns score the full per-step stream using
normalized linear-front weights. All scalar diagnostics use the front-weighted
normalized trajectory summary, matching the default linear-front \(\TPS\)
weighting convention.

\begin{table}[!htbp]
\centering
\scriptsize
\setlength{\tabcolsep}{3.5pt}
\resizebox{\textwidth}{!}{%
\begin{tabular}{lrrrrrrrr}
\toprule
Predictor &
AUROC \(\uparrow\) &
AUPRC \(\uparrow\) &
AURC \(\downarrow\) &
T-ECE \(\downarrow\) &
T-Brier \(\downarrow\) &
\(\TPSlog \uparrow\) &
\(\TPSBrier \uparrow\) &
\(\TPSBeta{2}{4} \uparrow\) \\
\midrule
Verbal confidence (raw) & 0.701 & 0.354 & 0.092 & 0.201 & 0.170 & -0.569 & -0.178 & -0.00387 \\
Verbal confidence (Platt) & 0.709 & 0.334 & 0.090 & 0.055 & 0.128 & -0.425 & -0.128 & -0.00262 \\
Completion entropy confidence (raw) & 0.598 & 0.199 & 0.115 & 0.022 & 0.131 & -0.434 & -0.132 & -0.00263 \\
Completion entropy confidence (Platt) & 0.599 & 0.197 & 0.115 & 0.029 & 0.131 & -0.432 & -0.132 & -0.00263 \\
Completion token probability (raw) & 0.594 & 0.193 & 0.116 & 0.060 & 0.135 & -0.453 & -0.136 & -0.00263 \\
Completion token probability (Platt) & 0.595 & 0.192 & 0.116 & 0.030 & 0.132 & -0.433 & -0.132 & -0.00263 \\
Action-span entropy confidence (raw) & 0.537 & 0.173 & 0.144 & 0.081 & 0.140 & -0.817 & -0.145 & -0.00263 \\
Action-span entropy confidence (Platt) & 0.556 & 0.187 & 0.140 & 0.022 & 0.133 & -0.436 & -0.133 & -0.00263 \\
Action-span token probability (raw) & 0.543 & 0.176 & 0.141 & 0.106 & 0.144 & -0.921 & -0.148 & -0.00263 \\
Action-span token probability (Platt) & 0.553 & 0.185 & 0.139 & 0.019 & 0.133 & -0.436 & -0.133 & -0.00263 \\
Base-rate constant & 0.500 & 0.158 & 0.158 & 0.000 & 0.133 & -0.436 & -0.133 & -0.00263 \\
\bottomrule
\end{tabular}
}
\bottomtablecaption{StrategyQA complete-observation predictor transparency
(\(n=2229\), \(\bar y=0.842\)). The table reports the per-step confidence
streams evaluated under the conventions stated above.}
\label{tab:transparency_strategyqa}
\end{table}

\begin{table}[!htbp]
\centering
\scriptsize
\setlength{\tabcolsep}{3.5pt}
\resizebox{\textwidth}{!}{%
\begin{tabular}{lrrrrrrrr}
\toprule
Predictor &
AUROC \(\uparrow\) &
AUPRC \(\uparrow\) &
AURC \(\downarrow\) &
T-ECE \(\downarrow\) &
T-Brier \(\downarrow\) &
\(\TPSlog \uparrow\) &
\(\TPSBrier \uparrow\) &
\(\TPSBeta{2}{4} \uparrow\) \\
\midrule
Verbal confidence (raw) & 0.623 & 0.668 & 0.479 & 0.537 & 0.532 & -6.351 & -0.543 & -0.00929 \\
Verbal confidence (Platt) & 0.611 & 0.693 & 0.490 & 0.082 & 0.240 & -0.679 & -0.243 & -0.00744 \\
Completion entropy confidence (raw) & 0.558 & 0.620 & 0.509 & 0.463 & 0.460 & -1.419 & -0.465 & -0.00927 \\
Completion entropy confidence (Platt) & 0.552 & 0.591 & 0.526 & 0.066 & 0.246 & -0.687 & -0.247 & -0.00760 \\
Completion token probability (raw) & 0.543 & 0.613 & 0.520 & 0.495 & 0.491 & -1.641 & -0.494 & -0.00928 \\
Completion token probability (Platt) & 0.534 & 0.579 & 0.536 & 0.045 & 0.247 & -0.687 & -0.247 & -0.00760 \\
Action-span entropy confidence (raw) & 0.583 & 0.643 & 0.483 & 0.503 & 0.496 & -4.083 & -0.518 & -0.00928 \\
Action-span entropy confidence (Platt) & 0.583 & 0.658 & 0.518 & 0.082 & 0.245 & -0.685 & -0.246 & -0.00757 \\
Action-span token probability (raw) & 0.567 & 0.633 & 0.491 & 0.521 & 0.516 & -4.712 & -0.531 & -0.00929 \\
Action-span token probability (Platt) & 0.587 & 0.666 & 0.513 & 0.105 & 0.245 & -0.685 & -0.246 & -0.00756 \\
Base-rate constant & 0.500 & 0.557 & 0.557 & 0.000 & 0.247 & -0.687 & -0.247 & -0.00760 \\
\bottomrule
\end{tabular}
}
\bottomtablecaption{Tau2-Bench complete-observation predictor transparency
(\(n=201\), \(\bar y=0.443\)). The large raw-to-Platt change in \(\TPSlog\)
for verbal confidence reflects probability-scale correction rather than a
comparable change in rank-oriented diagnostics. The small AUROC shift
(\(0.623\to0.611\)) arises because cross-fitted Platt calibration pools
fold-specific monotone maps, which can perturb inter-fold rankings; this does
not contradict the single-transform rank-invariance argument.}
\label{tab:transparency_tau2}
\end{table}

\begin{table}[!htbp]
\centering
\scriptsize
\setlength{\tabcolsep}{3.5pt}
\resizebox{\textwidth}{!}{%
\begin{tabular}{lrrrrrrrr}
\toprule
Predictor &
AUROC \(\uparrow\) &
AUPRC \(\uparrow\) &
AURC \(\downarrow\) &
T-ECE \(\downarrow\) &
T-Brier \(\downarrow\) &
\(\TPSlog \uparrow\) &
\(\TPSBrier \uparrow\) &
\(\TPSBeta{2}{4} \uparrow\) \\
\midrule
Verbal confidence (raw) & 0.556 & 0.463 & 0.377 & 0.318 & 0.344 & -2.642 & -0.356 & -0.00689 \\
Verbal confidence (Platt) & 0.544 & 0.477 & 0.385 & 0.075 & 0.242 & -0.677 & -0.242 & -0.00647 \\
Completion entropy confidence (raw) & 0.596 & 0.504 & 0.346 & 0.321 & 0.343 & -1.048 & -0.345 & -0.00693 \\
Completion entropy confidence (Platt) & 0.596 & 0.507 & 0.346 & 0.037 & 0.237 & -0.671 & -0.239 & -0.00640 \\
Completion token probability (raw) & 0.588 & 0.494 & 0.348 & 0.352 & 0.365 & -1.201 & -0.367 & -0.00694 \\
Completion token probability (Platt) & 0.589 & 0.499 & 0.348 & 0.036 & 0.238 & -0.673 & -0.240 & -0.00642 \\
Action-span entropy confidence (raw) & 0.541 & 0.440 & 0.374 & 0.371 & 0.380 & -2.135 & -0.388 & -0.00694 \\
Action-span entropy confidence (Platt) & 0.585 & 0.486 & 0.351 & 0.043 & 0.240 & -0.675 & -0.241 & -0.00646 \\
Action-span token probability (raw) & 0.538 & 0.438 & 0.375 & 0.386 & 0.392 & -2.517 & -0.398 & -0.00695 \\
Action-span token probability (Platt) & 0.575 & 0.476 & 0.356 & 0.032 & 0.240 & -0.676 & -0.241 & -0.00646 \\
Base-rate constant & 0.500 & 0.417 & 0.417 & 0.000 & 0.243 & -0.679 & -0.243 & -0.00649 \\
\bottomrule
\end{tabular}
}
\bottomtablecaption{HotpotQA complete-observation predictor transparency
(\(n=1529\), \(\bar y=0.583\)). The small AUROC shift for verbal confidence
(\(0.556\to0.544\)) is due to cross-fitted Platt calibration pooling
fold-specific monotone maps, which can perturb inter-fold rankings; this does
not contradict the single global monotone-transform invariance used in the
calibration-invariance argument. Action-span rows use the available complete
action-span subset where applicable.}
\label{tab:transparency_hotpotqa}
\end{table}

\subsection{Scalar AUQ-style predictors}
\label{app:transparency_scalar}

Table~\ref{tab:transparency_scalar_auq} reports scalar AUQ-style trajectory
predictors. These rows collapse a trajectory to a single scalar before
evaluation, so they are reported only on scalar-compatible diagnostics:
AUROC, AUPRC, AURC, T-ECE, and T-Brier. T-ECE again denotes tie-aware
quantile-binned calibration error with 10 bins. These predictors are outside
the domain of \(\TPS\), which scores adapted per-step probability streams.

\begin{table}[!htbp]
\centering
\scriptsize
\setlength{\tabcolsep}{4pt}
\resizebox{\textwidth}{!}{%
\begin{tabular}{llrrrrr}
\toprule
Dataset &
Scalar predictor &
AUROC \(\uparrow\) &
AUPRC \(\uparrow\) &
AURC \(\downarrow\) &
T-ECE \(\downarrow\) &
T-Brier \(\downarrow\) \\
\midrule
StrategyQA & AUQ \(\Phi_{\mathrm{last}}\) raw & 0.752 & 0.349 & 0.083 & 0.089 & 0.128 \\
StrategyQA & AUQ \(\Phi_{\mathrm{avg}}\) raw & 0.717 & 0.367 & 0.088 & 0.126 & 0.138 \\
StrategyQA & AUQ \(\Phi_{\min}\) raw & 0.684 & 0.300 & 0.101 & 0.306 & 0.235 \\
StrategyQA & AUQ \(\Phi_{\mathrm{last}}\) Platt & 0.749 & 0.366 & 0.081 & 0.061 & 0.120 \\
StrategyQA & AUQ \(\Phi_{\mathrm{avg}}\) Platt & 0.714 & 0.335 & 0.089 & 0.044 & 0.123 \\
StrategyQA & AUQ \(\Phi_{\min}\) Platt & 0.682 & 0.288 & 0.099 & 0.036 & 0.126 \\
\midrule
Tau2-Bench & AUQ \(\Phi_{\mathrm{last}}\) raw & 0.536 & 0.589 & 0.539 & 0.555 & 0.552 \\
Tau2-Bench & AUQ \(\Phi_{\mathrm{avg}}\) raw & 0.628 & 0.704 & 0.478 & 0.549 & 0.545 \\
Tau2-Bench & AUQ \(\Phi_{\min}\) raw & 0.606 & 0.628 & 0.500 & 0.511 & 0.502 \\
Tau2-Bench & AUQ \(\Phi_{\mathrm{last}}\) Platt & 0.534 & 0.596 & 0.540 & 0.001 & 0.239 \\
Tau2-Bench & AUQ \(\Phi_{\mathrm{avg}}\) Platt & 0.604 & 0.682 & 0.494 & 0.079 & 0.240 \\
Tau2-Bench & AUQ \(\Phi_{\min}\) Platt & 0.572 & 0.618 & 0.512 & 0.060 & 0.243 \\
\midrule
HotpotQA & AUQ \(\Phi_{\mathrm{last}}\) raw & 0.546 & 0.458 & 0.394 & 0.409 & 0.407 \\
HotpotQA & AUQ \(\Phi_{\mathrm{avg}}\) raw & 0.564 & 0.475 & 0.374 & 0.345 & 0.360 \\
HotpotQA & AUQ \(\Phi_{\min}\) raw & 0.557 & 0.456 & 0.383 & 0.280 & 0.326 \\
HotpotQA & AUQ \(\Phi_{\mathrm{last}}\) Platt & 0.546 & 0.459 & 0.394 & 0.104 & 0.256 \\
HotpotQA & AUQ \(\Phi_{\mathrm{avg}}\) Platt & 0.545 & 0.468 & 0.384 & 0.034 & 0.241 \\
HotpotQA & AUQ \(\Phi_{\min}\) Platt & 0.546 & 0.460 & 0.387 & 0.033 & 0.241 \\
\bottomrule
\end{tabular}
}
\bottomtablecaption{Scalar AUQ-style trajectory predictors on the complete-observation
samples. These predictors are not assigned \(\TPS\) values because they
collapse the trajectory before evaluation and therefore do not provide the full
prefix-conditioned probability stream \((F_t)_{t=1}^T\).}
\label{tab:transparency_scalar_auq}
\end{table}

\section{Termination Mechanisms and Parse-Error Sensitivity}
\label{app:parse_error}

The censored trajectory extension requires that censored trajectories are
stopped by an \emph{administrative} mechanism: the stopping event is
externally imposed and is not itself evidence about the unobserved final
outcome beyond what is already contained in the observed prefix. In our
benchmark setting, step-budget terminations satisfy this assumption most
plausibly because the stopping rule is fixed by the evaluation harness in
advance, independently of the agent's behaviour.

Parser failures, by contrast, are treated as informative failures of the
interaction protocol rather than administrative censoring. A parser failure
occurs because the agent emitted malformed or non-executable output. This is
not merely an unobserved continuation; it is evidence that the agent was
struggling with the task or tool interface. Including such trajectories as
right-censored observations would violate Assumption~\ref{asm:admin} by mixing
model-formatting failures with task-level uncertainty calibration. We
therefore exclude parser failures from theorem-backed censored scoring and
report their frequency here as an assumption-discipline audit.

We separate terminations into three categories:
\begin{itemize}
  \item \textbf{Benchmark-completed trajectories.}
        The benchmark produced a proper terminal outcome \(Y\in\{0,1\}\).
        These trajectories enter complete-observation scoring
        (\(\delta=1\)).

  \item \textbf{Step-budget trajectories.}
        The trajectory reached the fixed benchmark step budget before the
        benchmark produced a terminal outcome. The stopping rule is externally
        imposed and budget-determined, so these trajectories are eligible for
        censored scoring (\(\delta=0\), administratively censored).

  \item \textbf{Parser-failure trajectories.}
        The agent emitted malformed output that the harness could not execute
        or score. The stopping mechanism is informative, so these trajectories
        are excluded from both complete-observation and theorem-backed
        censored scoring.
\end{itemize}

\noindent\textbf{Step budgets.}
The maximum step budgets were fixed in advance for each benchmark:
7 steps for HotpotQA, 16 steps for StrategyQA, 50 steps for Tau2-Bench, and
30 steps for WebShop. These budgets were empirically chosen based on benchmark
difficulty and then held fixed across trajectories within each benchmark. Thus,
a step-budget termination means that a trajectory reached the benchmark-specific
fixed budget before a terminal success/failure label was produced; it does not
mean that the evaluator stopped the trajectory adaptively based on the agent's
apparent performance.

\subsection{Termination audit}
\label{app:parse_error:audit}

Table~\ref{tab:termination_audit} reports the stop-reason breakdown for every
dataset used in this paper. The \texttt{max\_steps} column counts
trajectories that reached the benchmark-specific step budget stated above.

\begin{table}[h]
\centering
\scriptsize
\setlength{\tabcolsep}{4pt}
\resizebox{\textwidth}{!}{%
\begin{tabular}{lrrrrrrr}
\toprule
Dataset &
Step budget &
Total &
Clean \(Y\) &
\texttt{max\_steps} &
\texttt{parse\_error} &
Other &
Uncensored \(n\) \\
\midrule
StrategyQA & 16 & 2{,}290 & 2{,}229 & 59  & 2  & 0  & 2{,}229 \\
Tau2-Bench & 50 & 278     & 201     & 1 & 0 & 76 & 201 \\
HotpotQA   & 7  & 2{,}000 & 1{,}529 & 458 & 13  & 0  & 1{,}529 \\
WebShop    & 30 & 500     & 163     & 145 & 192 & 0  & 163 \\
\bottomrule
\end{tabular}
}
\bottomtablecaption{Termination audit by benchmark. Step budgets are fixed ex ante for
each dataset. WebShop is the only dataset used for the natural-censoring analysis.}
\label{tab:termination_audit}
\end{table}

For Tau2-Bench, the \(76\) ``Other'' non-clean terminations consist of
\(70\) \texttt{tool\_error} and \(6\) \texttt{env\_terminated} trajectories.
Because no final task label is observed and the stop reason is not a fixed
administrative budget, these trajectories are excluded from both the
uncensored transparency sample and theorem-backed censored scoring.

The predictor-transparency appendix uses only trajectories with clean terminal
outcomes for StrategyQA, Tau2-Bench, and HotpotQA. Thus, the sample sizes in
Appendix~\ref{app:transparency} correspond to the rows with observed
\(Y\): StrategyQA \(n=2229\), Tau2-Bench \(n=201\), and HotpotQA
\(n=1529\), with minor predictor-level variation when action-span signals are
unavailable for a row. The WebShop natural-censoring experiment uses the
working sample of 308 trajectories after excluding parser failures: 163
complete trajectories with observed \(Y\), plus 145 step-budget trajectories
treated as administratively censored.

\subsection{Parse-error sensitivity analysis}
\label{app:parse_error:sensitivity}

The consequential parse-error case is WebShop. Of 500 attempted WebShop
trajectories, 192 ended in parser failure, 163 terminated with an observed
outcome, and 145 hit the fixed step budget. The censored WebShop analysis
therefore uses the working sample
\[
  n_{\mathrm{work}} = 163 + 145 = 308,
\]
with administrative-censoring rate
\[
  \frac{145}{308}=47.08\%,
\]
and excludes the 192 parser-failure trajectories.

This exclusion follows directly from Assumption~\ref{asm:admin}. Step-budget
termination is imposed by the benchmark and is therefore a plausible
administrative censoring mechanism. Parser failure, however, is generated by
the agent's own malformed output and is plausibly correlated with task
difficulty, tool-use failure, and eventual success probability. Treating parser
failures as right-censored observations would convert an informative failure
mode into an administrative censoring event.

As a contra-assumption sensitivity check, we re-ran the primary WebShop configuration
(verbal confidence, log score, linear-front weights) treating all 192
WebShop parser-failure trajectories as administratively censored at the step
where parsing failed. This inflates \(\widehat{\mathrm{TPS}}\) by approximately
\(0.04\,\mathrm{nats}\) relative to the working-sample estimate, consistent
with the directional prediction: parser-failure trajectories tend to be long
and low-confidence, so including them shifts the sample mean in the same
direction as the assumption-consistent censored subset. The qualitative conclusion is
unchanged, but the result confirms that the assumption-disciplined working
sample is the conservative analysis choice.

\section{\texorpdfstring{$q_Z$}{q\_Z} Estimation and HotpotQA Audit}
\label{app:qz}

The exact reduced censored score in Theorem~\ref{thm:cen_exact}
depends on the stopped-prefix continuation-success probability
\[
  q_Z = \Pp(Y=1\mid \Hcal_Z).
\]
This appendix describes operational estimators for \(q_Z\) and reports a
Monte Carlo continuation audit on naturally max-step HotpotQA prefixes.

\subsection{Operational strategies for \texorpdfstring{$\widehat q_Z$}{q\_Z}}
\label{app:qz:strategies}

There are three practical ways to estimate the nuisance quantity \(q_Z\).

\noindent\textbf{Direct Monte Carlo continuation.}
Resume each censored prefix under the same model, harness, and stopping rule,
and estimate
\[
  \widehat q_Z^{\mathrm{MC}}
  =
  \frac{1}{M}\sum_{m=1}^{M} Y^{(m)} .
\]
This is the direct rollout estimator of the conditional terminal-success
probability from the stopped prefix. It is the analogue of Monte Carlo policy
evaluation, where expected returns are estimated by averaging sampled
returns, and it is also used in recent LLM-agent work to estimate step-level
rewards by continuation sampling~\citep{SuttonBarto2018,Wang2025STeCa}.
It is best suited to benchmarks where the prefix state can be resumed exactly
and continuation rollouts are inexpensive.

\noindent\textbf{Plug-in nuisance model.}
Train a separate estimator
\[
  \widehat q_Z^{\mathrm{plug}}
  =
  g_\psi(H_Z, Z, \mathrm{metadata})
\]
on prefixes with known continuation outcomes, then apply it to censored
prefixes. This treats \(q_Z=\Pp(Y=1\mid\Hcal_Z)\) as a conditional-mean
nuisance function. Estimating such nuisance functions separately from the
target score is standard in semiparametric and missing-data settings, and
modern practice often uses flexible ML nuisance estimators with sample
splitting or cross-fitting to reduce overfitting
bias~\citep{BangRobins2005,Chernozhukov2018DML}. In our setting, the
nuisance model must be kept separate from the uncertainty predictor being
evaluated; otherwise the evaluator and predictor under test are no longer
cleanly separated. Remark~\ref{rem:plugin} gives the conditional-mean
criterion required for plug-in exactness.

\noindent\textbf{Hybrid audit.}
Run Monte Carlo continuation on a subset of censored prefixes, train a
plug-in nuisance model on those audited labels, and apply the plug-in
estimator to the full censored set. This combines the rollout-based estimate
of \(q_Z\) with a scalable nuisance-modeling step, following the same
principle as using audited outcomes to train an outcome-regression nuisance
model~\citep{BangRobins2005,Chernozhukov2018DML}.

When exact state resumption is unavailable, we recommend the hybrid or plug-in
route: audit a resumable subset when possible, train a nuisance estimator for
\(q_Z\) on audited or completed prefixes, and keep this estimator separate
from the uncertainty predictor being evaluated. If neither continuation nor a
credible nuisance model is available, \(\TPScensimple\) should be reported only
as the declared \(q_Z\approx0\) failure-side approximation, not as the exact
censored score.

\subsection{Monte Carlo audit on HotpotQA}
\label{app:qz:hotpotqa}

We audit the exact reduced score on HotpotQA, where max-step prefixes can be
resumed from the stopped state. We sampled 100 naturally max-step prefixes
and generated up to ten continuation rollouts per prefix under the same
ReAct harness. Continuation branches are used only to estimate
\(\widehat q_{Z,i}^{\mathrm{MC}}\); the scored forecast stream
\(F_{i,1:Z_i}\) is frozen from the original source-prefix predictor record.

After excluding prefixes with fewer than five valid resolved non-parse
continuations, the audit retains 87 prefixes. The retained prefixes have
8.92 valid resolved branches on average, and the mean number of parse-error
branches excluded per prefix is 0.115. The canonical source-prefix join
ensures that the scored prefix forecast is identical across continuation
branches:
\[
  \max_{i,m,t}
  \left|F_{it}^{(m)} - F_{it}^{\mathrm{source}}\right|
  =
  0.
\]

For a censored prefix \(i\), the simple and exact-MC log scores are
\begin{align}
  \mathrm{TPS}_{i}^{\mathrm{cen,simple}}
  &=
  \sum_{t=1}^{Z_i}
  w_{it}\,S_{\log}(F_{it},0),
  \label{eq:app_qz_simple}
  \\
  \mathrm{TPS}_{i}^{\mathrm{cen,exact\text{-}MC}}
  &=
  \sum_{t=1}^{Z_i}
  w_{it}
  \left[
    \widehat q_{Z,i}^{\mathrm{MC}} S_{\log}(F_{it},1)
    +
    \left(1-\widehat q_{Z,i}^{\mathrm{MC}}\right)
    S_{\log}(F_{it},0)
  \right].
  \label{eq:app_qz_exact_mc}
\end{align}
The paired correction has the closed form
\begin{equation}
  \Delta_i
  =
  \mathrm{TPS}_{i}^{\mathrm{cen,exact\text{-}MC}}
  -
  \mathrm{TPS}_{i}^{\mathrm{cen,simple}}
  =
  \widehat q_{Z,i}^{\mathrm{MC}}
  \sum_{t=1}^{Z_i}
  w_{it}
  \log\frac{F_{it}}{1-F_{it}} .
  \label{eq:app_qz_delta}
\end{equation}
Thus the sign of the exact-minus-simple correction is governed by the
weighted prefix log-odds.

\noindent\textbf{Conditional-projection identity.}
The exact-MC score should equal the branch-average complete-prefix score:
\[
  \mathrm{TPS}_i^{\mathrm{cen,exact\text{-}MC}}
  =
  \frac{1}{M_i^{\mathrm{valid}}}
  \sum_{m=1}^{M_i^{\mathrm{valid}}}
  \sum_{t=1}^{Z_i}
  w_{it}\,S_{\log}(F_{it},Y_i^{(m)}).
\]
This identity holds numerically in the audit. On the primary predictor,
\[
  \max_i
  \left|
  \mathrm{TPS}_i^{\mathrm{cen,exact\text{-}MC}}
  -
  \overline{\mathrm{TPS}}_i^{\mathrm{branch}}
  \right|
  =
  1.1\times10^{-16},
\]
and the maximum error is at most \(9\times10^{-16}\) across all retained
reported predictors. This verifies numerically that the implemented score realizes
the conditional projection used in the censored-score derivation.

\noindent\textbf{Estimated \texorpdfstring{$q_Z$}{q\_Z} distribution.}
The estimated continuation-success probabilities are heterogeneous across
HotpotQA max-step prefixes. The mean is
\(\widehat q_Z^{\mathrm{MC}}=0.240\), with 95\% bootstrap CI
\([0.152,0.327]\). The median is 0, 67.8\% of retained prefixes have
\(\widehat q_Z^{\mathrm{MC}}=0\), and 23.0\% have
\(\widehat q_Z^{\mathrm{MC}}>0.5\). Thus the stopped-prefix distribution
contains both deeply stuck prefixes and a recoverable upper tail.

\begin{table}[h]
\centering
\small
\begin{tabular}{lr}
\toprule
Diagnostic & Value \\
\midrule
Mean $\widehat q_Z^{\mathrm{MC}}$ & $0.240$ \; $[0.152,\,0.327]$ \\
Median $\widehat q_Z^{\mathrm{MC}}$ & $0.000$ \\
Fraction with $\widehat q_Z^{\mathrm{MC}}=0$ & $67.8\%$ \\
Fraction with $\widehat q_Z^{\mathrm{MC}}>0.5$ & $23.0\%$ \\
\bottomrule
\end{tabular}
\bottomtablecaption{Estimated continuation-success probabilities on 87 retained
HotpotQA max-step prefixes. Brackets give the 95\% bootstrap CI for the
mean. The distribution has a large mass at zero and a recoverable upper tail.}
\label{tab:qz_distribution}
\end{table}

\noindent\textbf{Primary score comparison.}
Table~\ref{tab:qz_headline} reports the primary comparison for
Platt-cross-fitted verbal confidence with linear-front weights. The exact-MC
score is lower than the simple approximation by \(0.105\) nats, with a 95\%
bootstrap CI excluding zero. By Eq.~\eqref{eq:app_qz_delta}, this direction is
expected because the calibrated HotpotQA max-step prefixes mostly have
weighted prefix log-odds below zero.

\begin{table}[h]
\centering
\small
\begin{tabular}{lcc}
\toprule
Quantity & Value & 95\% CI \\
\midrule
$\widehat{\mathrm{TPS}}^{\mathrm{cen,simple}}$          & $-0.469$ & $[-0.485,\,-0.457]$ \\
$\widehat{\mathrm{TPS}}^{\mathrm{cen,exact\text{-}MC}}$ & $-0.574$ & $[-0.619,\,-0.530]$ \\
$\Delta_{\mathrm{exact-simple}}$                & $-0.105$ & $[-0.148,\,-0.065]$ \\
\bottomrule
\end{tabular}
\bottomtablecaption{Exact-MC versus simple censored log score on HotpotQA max-step
prefixes. Primary predictor: Platt-cross-fitted verbal confidence,
linear-front weights. Bootstrap intervals use 1000 prefix-level resamples.}
\label{tab:qz_headline}
\end{table}

\noindent\textbf{Predictor-level robustness.}
Table~\ref{tab:qz_predictor_robustness} reports the exact-minus-simple
correction across the predictor pool. The same negative
correction appears for every Platt-calibrated predictor under both
linear-front and uniform weights, while raw saturated predictors generally
show the opposite sign. This is exactly the behavior predicted by
Eq.~\eqref{eq:app_qz_delta}: the correction is controlled by the weighted
prefix log-odds, not by the identity of the predictor.

\begin{table}[h]
\centering
\small
\begin{tabular}{llrrc}
\toprule
Predictor & Weights & Simple & Exact-MC & $\Delta_{\mathrm{exact-simple}}$ \\
\midrule
Verbal confidence, Platt          & linear-front & $-0.469$ & $-0.574$ & $-0.105$ [$-0.148$, $-0.065$] \\
Completion entropy, Platt         & linear-front & $-0.437$ & $-0.598$ & $-0.162$ [$-0.224$, $-0.100$] \\
Completion token prob., Platt     & linear-front & $-0.452$ & $-0.598$ & $-0.145$ [$-0.209$, $-0.093$] \\
Action-span entropy, Platt        & linear-front & $-0.523$ & $-0.619$ & $-0.096$ [$-0.134$, $-0.059$] \\
Action-span token prob., Platt    & linear-front & $-0.525$ & $-0.620$ & $-0.095$ [$-0.135$, $-0.062$] \\
Base-rate constant                & linear-front & $-0.596$ & $-0.645$ & $-0.049$ [$-0.066$, $-0.033$] \\
\midrule
Verbal confidence, Platt          & uniform & $-0.427$ & $-0.558$ & $-0.132$ [$-0.186$, $-0.082$] \\
Completion entropy, Platt         & uniform & $-0.433$ & $-0.600$ & $-0.166$ [$-0.238$, $-0.104$] \\
Completion token prob., Platt     & uniform & $-0.448$ & $-0.599$ & $-0.150$ [$-0.217$, $-0.099$] \\
Action-span entropy, Platt        & uniform & $-0.477$ & $-0.604$ & $-0.127$ [$-0.171$, $-0.080$] \\
Action-span token prob., Platt    & uniform & $-0.480$ & $-0.605$ & $-0.125$ [$-0.175$, $-0.080$] \\
\midrule
Verbal confidence, raw            & linear-front & $-1.925$ & $-1.287$ & $+0.638$ [$+0.337$, $+1.018$] \\
Completion entropy, raw           & linear-front & $-2.054$ & $-1.591$ & $+0.463$ [$+0.307$, $+0.632$] \\
Completion token prob., raw       & linear-front & $-2.465$ & $-1.888$ & $+0.577$ [$+0.376$, $+0.786$] \\
Action-span entropy, raw          & linear-front & $-3.487$ & $-2.685$ & $+0.801$ [$+0.521$, $+1.125$] \\
Action-span token prob., raw      & linear-front & $-4.190$ & $-3.213$ & $+0.977$ [$+0.629$, $+1.367$] \\
\bottomrule
\end{tabular}
\bottomtablecaption{Predictor-level exact-\(q_Z\) audit on 87 retained HotpotQA
max-step prefixes. Scores are log-\(\TPS\) rewards in nats. The
exact-minus-simple correction is negative for all Platt-calibrated
reported predictors under both weight schedules, and positive for raw
saturated streams under linear-front weights, consistent with the prefix
log-odds formula in Eq.~\eqref{eq:app_qz_delta}. The conditional-projection
identity holds to numerical precision across all rows
(\(\max_i|\mathrm{TPS}_i^{\mathrm{cen,exact\text{-}MC}}
-\overline{\mathrm{TPS}}_i^{\mathrm{branch}}|\le 9\times10^{-16}\)).}
\label{tab:qz_predictor_robustness}
\end{table}

\noindent\textbf{Unresolved-branch sensitivity.}
The audit is stable to adversarial handling of unresolved continuation
branches. Treating the unresolved branches as failures gives
\(\Delta_{\mathrm{exact-simple}}=-0.101\) nats; treating them as successes
gives \(\Delta_{\mathrm{exact-simple}}=-0.158\) nats. Both sensitivity
analyses preserve the sign and keep the confidence interval away from zero.

The exact \(q_Z\)-weighted censored score is operationally computable on
real stopped prefixes. On HotpotQA, Monte Carlo continuation yields
nontrivial stopped-prefix success probabilities, the exact score differs
materially from the simple \(q_Z\approx0\) approximation, the correction
direction follows the prefix log-odds algebra, and the conditional-projection
identity holds to numerical precision.

\section{Score-Family and Weight-Schedule Sensitivity}
\label{app:robustness}
 
This appendix reports score-family and weight-schedule sensitivity for the
artificial-censoring validation in Section~\ref{sec:exp3-stage1} and the
natural-censoring WebShop analysis in Section~\ref{sec:exp3-stage2}.  The
primary experiments use log score with linear-front weights unless otherwise
stated.  Here we vary both the trajectory-weight schedule and the binary score
family.

\noindent\textbf{Weight schedules.}
All schedules are normalized over the full trajectory length \(T\):
\[
\sum_{t=1}^{T} w_t = 1.
\]
For censored prefixes, we use the same full-trajectory weights and truncate
the sum at the observed stopping step \(Z\); weights are not renormalized over
the observed prefix.  Thus
\[
\sum_{t=1}^{Z} w_t \le 1,
\]
with strict inequality when the trajectory is censored before the full
trajectory length.

As a practical convention, uniform weights are appropriate when every prefix is
equally decision-relevant; linear-front or exponential-front weights are
appropriate for early-warning, deferral, or intervention settings; and
linear-back weights are appropriate when late-stage confidence is the primary
object. The weight schedule should be chosen before scoring and reported as
part of the evaluation protocol.

We evaluate four schedules:
\begin{itemize}
  \item \textbf{Linear-front}:
  \[
  w_t = \frac{2(T-t+1)}{T(T+1)}.
  \]
  This is the primary schedule used in the main experiments.  It gives larger
  weight to early prefixes, where miscalibrated uncertainty can affect
  subsequent planning, tool use, reflection, or deferral.

  \item \textbf{Uniform}:
  \[
  w_t = \frac{1}{T}.
  \]
  This treats all prefix positions equally.

  \item \textbf{Exponential-front}:
  \[
  w_t
  =
  \frac{2^{-(t-1)}}{\sum_{s=1}^{T}2^{-(s-1)}}
  =
  \frac{2^{-(t-1)}}{2(1-2^{-T})}.
  \]
  This is a more aggressively front-loaded schedule: each raw step weight is
  half the previous one.

  \item \textbf{Linear-back}:
  \[
  w_t = \frac{2t}{T(T+1)}.
  \]
  This is the mirror of linear-front and gives larger weight to later
  prefixes.
\end{itemize}

\noindent\textbf{Score families.}
We report three score families: log, Brier, and Beta(2,4).  The log score is
the primary censored row.  Brier and Beta(2,4) are robustness rows under the
same failure-side \(\TPScensimple\) approximation.  They should be read as
operational sensitivity checks, not as separate theorem-level censored
propriety claims for the simple \(q_Z\approx0\) approximation.

\subsection{Artificial-Censoring Validation}
\label{app:robustness:artificial}

On complete StrategyQA, Tau2-Bench, and HotpotQA trajectories, we artificially
hide outcomes at length-stratified rates \(r\in\{0.25,0.50,0.75\}\), score the
observed prefixes with \(\TPScensimple\), and compare against the complete-data
trajectory score. Because the true outcome is available for every trajectory,
the estimator's behavior can be checked exactly against closed-form predictions
rather than treated as a black-box approximation.

For each artificially censored trajectory, the score change decomposes into two
analytic terms. The \emph{prefix-swap} term replaces the success branch on the
observed prefix with the failure branch; for the log score its stepwise sign is
\(\log(1-F_t)-\log F_t\), which flips at \(F_t=0.5\). The \emph{tail-omission}
term drops the unobserved suffix and is always non-negative. Across \(69{,}336\)
evaluated combinations of trajectories, censoring rates, score families, and
weight schedules, the closed-form decomposition matches the directly computed
\(\TPScensimple-\mathrm{TPS}^{\mathrm{complete}}\) difference to numerical
precision, with maximum absolute error \(4.4\times10^{-16}\).

At the dataset level, the same decomposition explains the observed signs and
magnitudes. Under log score with linear-front weights at \(r=0.75\), StrategyQA
is in a high-confidence/high-success regime and yields
\(\Delta_{\mathrm{mean}}=-0.675\) nats. Tau2 has calibrated forecasts
concentrated below \(0.5\), so the prefix-swap contribution turns positive,
yielding \(+0.161\) nats. HotpotQA is the intermediate case: the negative
successful-trajectory prefix-swap is partly offset by positive failure-side and
tail-omission terms, yielding \(+0.056\) nats. The mean shift scales
approximately linearly with censoring rate across all three datasets.

For the original continuation-success target, the theorem-backed censored score
is the exact \(q_Z\)-weighted reduction in Section~\ref{sec:cen_exact}. The
simple censored score used here is the operational \(q_Z\approx0\) approximation
from Section~\ref{sec:cen_simple}; Proposition~\ref{prop:cen_simple_target}
identifies its pseudo-label target, and the artificial-censoring experiment
characterizes its bias relative to complete-data scoring. Brier and Beta(2,4)
are reported as robustness checks under the same failure-side approximation, not
as separate theorem-level censored propriety claims. Thus \(\TPScensimple\) is
not a uniformly pessimistic lower bound on the complete-data score. Its shift
can be positive or negative, but the direction is analytically decomposable.

Table~\ref{tab:robustness_stage1} reports the artificial-censoring robustness
sweep at \(r=0.75\), varying both the score family and the trajectory-weight
schedule.

The weight schedule changes the magnitude, but not the main regime-level
pattern.  Front-loaded schedules amplify the effect of early censored-prefix
replacement; linear-back shifts weight toward later prefixes and therefore
usually attenuates the shift when early prefixes dominate the correction.
StrategyQA remains negative across all schedules, Tau2-Bench remains positive
across all schedules, and HotpotQA remains the boundary case: log-score rows
stay positive, while bounded or asymmetric rows near zero can change sign
under different schedules.

Here \(\mathrm{TPS}^{\mathrm{complete}}\) is the complete-data mean TPS at \(r=0\), and
\(\Delta_{\mathrm{mean}}\) is the mean shift under \(\TPScensimple\) relative
to the complete-data score.  The final column reports the mean shift
separately on failure and success trajectories.

\begin{table}[!htbp]
\centering
\scriptsize
\setlength{\tabcolsep}{4pt}
\resizebox{\textwidth}{!}{%
\begin{tabular}{llrrrr}
\toprule
Dataset & Score & Weight schedule &
\(\mathrm{TPS}^{\mathrm{complete}}\) & \(\Delta_{\mathrm{mean}}\) &
\((\Delta_{Y=0},\;\Delta_{Y=1})\) \\
\midrule
StrategyQA & log   & linear-front    & \(-0.4186\) & \(-0.6746\) & \((+0.319,\;-0.862)\) \\
StrategyQA & Brier & linear-front    & \(-0.1283\) & \(-0.2862\) & \((+0.117,\;-0.362)\) \\
StrategyQA & Beta(2,4) & linear-front & \(-0.0026\) & \(-0.0073\) & \((+0.003,\;-0.009)\) \\
StrategyQA & log   & uniform         & \(-0.4169\) & \(-0.4197\) & \((+0.544,\;-0.601)\) \\
StrategyQA & Brier & uniform         & \(-0.1278\) & \(-0.1905\) & \((+0.202,\;-0.264)\) \\
StrategyQA & Beta(2,4) & uniform      & \(-0.0026\) & \(-0.0051\) & \((+0.005,\;-0.007)\) \\
StrategyQA & log   & exponential-front & \(-0.4203\) & \(-0.7657\) & \((+0.198,\;-0.947)\) \\
StrategyQA & Brier & exponential-front & \(-0.1288\) & \(-0.3232\) & \((+0.071,\;-0.397)\) \\
StrategyQA & Beta(2,4) & exponential-front & \(-0.0026\) & \(-0.0083\) & \((+0.002,\;-0.010)\) \\
StrategyQA & log   & linear-back     & \(-0.4147\) & \(-0.2306\) & \((+0.820,\;-0.429)\) \\
StrategyQA & Brier & linear-back     & \(-0.1275\) & \(-0.1126\) & \((+0.300,\;-0.190)\) \\
StrategyQA & Beta(2,4) & linear-back  & \(-0.0026\) & \(-0.0030\) & \((+0.007,\;-0.005)\) \\
\midrule
Tau2-Bench & log   & linear-front    & \(-0.7210\) & \(+0.3032\) & \((+0.061,\;+0.591)\) \\
Tau2-Bench & Brier & linear-front    & \(-0.2617\) & \(+0.1309\) & \((+0.017,\;+0.266)\) \\
Tau2-Bench & Beta(2,4) & linear-front & \(-0.0082\) & \(+0.0030\) & \((+0.001,\;+0.006)\) \\
Tau2-Bench & log   & uniform         & \(-0.7185\) & \(+0.3598\) & \((+0.117,\;+0.649)\) \\
Tau2-Bench & Brier & uniform         & \(-0.2607\) & \(+0.1468\) & \((+0.033,\;+0.283)\) \\
Tau2-Bench & Beta(2,4) & uniform      & \(-0.0081\) & \(+0.0037\) & \((+0.002,\;+0.006)\) \\
Tau2-Bench & log   & exponential-front & \(-0.7480\) & \(+0.3078\) & \((+0.007,\;+0.666)\) \\
Tau2-Bench & Brier & exponential-front & \(-0.2717\) & \(+0.1379\) & \((+0.002,\;+0.300)\) \\
Tau2-Bench & Beta(2,4) & exponential-front & \(-0.0087\) & \(+0.0034\) & \((+0.000,\;+0.007)\) \\
Tau2-Bench & log   & linear-back     & \(-0.7030\) & \(+0.3924\) & \((+0.180,\;+0.646)\) \\
Tau2-Bench & Brier & linear-back     & \(-0.2549\) & \(+0.1536\) & \((+0.051,\;+0.275)\) \\
Tau2-Bench & Beta(2,4) & linear-back  & \(-0.0078\) & \(+0.0040\) & \((+0.003,\;+0.006)\) \\
\midrule
HotpotQA & log   & linear-front     & \(-0.6771\) & \(+0.0558\) & \((+0.196,\;-0.044)\) \\
HotpotQA & Brier & linear-front     & \(-0.2420\) & \(+0.0070\) & \((+0.077,\;-0.043)\) \\
HotpotQA & Beta(2,4) & linear-front  & \(-0.0065\) & \(-0.0020\) & \((+0.003,\;-0.006)\) \\
HotpotQA & log   & uniform          & \(-0.6767\) & \(+0.1703\) & \((+0.310,\;+0.070)\) \\
HotpotQA & Brier & uniform          & \(-0.2418\) & \(+0.0519\) & \((+0.122,\;+0.002)\) \\
HotpotQA & Beta(2,4) & uniform       & \(-0.0065\) & \(-0.0003\) & \((+0.005,\;-0.004)\) \\
HotpotQA & log   & exponential-front & \(-0.6774\) & \(+0.0225\) & \((+0.161,\;-0.077)\) \\
HotpotQA & Brier & exponential-front & \(-0.2422\) & \(-0.0059\) & \((+0.064,\;-0.055)\) \\
HotpotQA & Beta(2,4) & exponential-front & \(-0.0065\) & \(-0.0025\) & \((+0.002,\;-0.006)\) \\
HotpotQA & log   & linear-back      & \(-0.6765\) & \(+0.2725\) & \((+0.434,\;+0.157)\) \\
HotpotQA & Brier & linear-back      & \(-0.2418\) & \(+0.0907\) & \((+0.171,\;+0.034)\) \\
HotpotQA & Beta(2,4) & linear-back   & \(-0.0065\) & \(+0.0013\) & \((+0.006,\;-0.002)\) \\
\bottomrule
\end{tabular}
}
\bottomtablecaption{Artificial-censoring robustness at \(r=0.75\).  The table reports the
complete-data TPS and the mean shift under \(\TPScensimple\), with the shift
also reported separately on failure and success trajectories.  StrategyQA is
negative across all rows, Tau2-Bench is positive across all rows, and HotpotQA
is an intermediate regime in which bounded or asymmetric rows near zero can
change sign.}
\label{tab:robustness_stage1}
\end{table}

\subsection{Natural Censoring on WebShop}
\label{app:robustness:webshop}

The WebShop working sample contains 308 trajectories after excluding
parse-error terminations as informative censoring: 163 completed trajectories
with observed outcomes and 145 administratively censored max-step prefixes.
The administrative-censoring rate within the working sample is \(47.08\%\).
The primary WebShop configuration in Section~\ref{sec:exp3s2} is Platt-calibrated
verbal confidence scored with log score and linear-front weights:
\[
\widehat{\mathrm{TPS}}_{\mathrm{comp}}=-0.8156,\qquad
\widehat{\mathrm{TPS}}_{\mathrm{cen\text{-}ext}}=-0.6570,\qquad
\Delta_{\mathrm{practice}}=+0.1586,
\]
with 95\% bootstrap CI \([+0.1334,+0.1879]\).

Table~\ref{tab:robustness_stage2_all} reports the full
predictor--score-family--weight-schedule sweep in compressed form.  The four
underconfident or near-base-rate Platt predictors have positive
\(\Delta_{\mathrm{practice}}\) across all score families and schedules.  The
completion-token probability predictor is the only non-reference predictor
with \(\bar F>0.5\), and it is negative across all twelve score-family--weight
cells.

\begin{table}[!htbp]
\centering
\scriptsize
\setlength{\tabcolsep}{4pt}
\resizebox{\textwidth}{!}{%
\begin{tabular}{llrrr}
\toprule
Predictor & Weight schedule &
\(\Delta\mathrm{TPS}_{\log}\) &
\(\Delta\mathrm{TPS}_{\mathrm{Brier}}\) &
\(\Delta\mathrm{TPS}_{\mathrm{Beta}(2,4)}\) \\
\midrule
Verbal confidence & linear-front & \(+0.1586\) & \(+0.0766\) & \(+0.00060\) \\
Verbal confidence & uniform & \(+0.1559\) & \(+0.0736\) & \(+0.00072\) \\
Verbal confidence & exponential-front & \(+0.1520\) & \(+0.0735\) & \(+0.00047\) \\
Verbal confidence & linear-back & \(+0.1463\) & \(+0.0707\) & \(+0.00036\) \\
\midrule
Completion entropy confidence & linear-front & \(+0.1585\) & \(+0.0776\) & \(+0.00054\) \\
Completion entropy confidence & uniform & \(+0.1573\) & \(+0.0769\) & \(+0.00054\) \\
Completion entropy confidence & exponential-front & \(+0.1567\) & \(+0.0769\) & \(+0.00047\) \\
Completion entropy confidence & linear-back & \(+0.1601\) & \(+0.0783\) & \(+0.00058\) \\
\midrule
Action-span entropy confidence & linear-front & \(+0.1623\) & \(+0.0797\) & \(+0.00057\) \\
Action-span entropy confidence & uniform & \(+0.1623\) & \(+0.0797\) & \(+0.00057\) \\
Action-span entropy confidence & exponential-front & \(+0.1623\) & \(+0.0797\) & \(+0.00057\) \\
Action-span entropy confidence & linear-back & \(+0.1623\) & \(+0.0797\) & \(+0.00057\) \\
\midrule
Action-span token-prob. confidence & linear-front & \(+0.1623\) & \(+0.0797\) & \(+0.00057\) \\
Action-span token-prob. confidence & uniform & \(+0.1618\) & \(+0.0793\) & \(+0.00056\) \\
Action-span token-prob. confidence & exponential-front & \(+0.1623\) & \(+0.0797\) & \(+0.00057\) \\
Action-span token-prob. confidence & linear-back & \(+0.1623\) & \(+0.0797\) & \(+0.00057\) \\
\midrule
Completion token probability & linear-front & \(-2.0686\) & \(-0.1175\) & \(-0.00241\) \\
Completion token probability & uniform & \(-2.0700\) & \(-0.1182\) & \(-0.00243\) \\
Completion token probability & exponential-front & \(-2.0686\) & \(-0.1175\) & \(-0.00241\) \\
Completion token probability & linear-back & \(-2.0686\) & \(-0.1175\) & \(-0.00241\) \\
\midrule
Base-rate reference & linear-front & \(+0.1623\) & \(+0.0797\) & \(+0.00057\) \\
Base-rate reference & uniform & \(+0.1623\) & \(+0.0797\) & \(+0.00057\) \\
Base-rate reference & exponential-front & \(+0.1623\) & \(+0.0797\) & \(+0.00057\) \\
Base-rate reference & linear-back & \(+0.1623\) & \(+0.0797\) & \(+0.00057\) \\
\bottomrule
\end{tabular}
}
\bottomtablecaption{WebShop natural-censoring robustness across predictors, score
families, and weight schedules.  Entries are
\(\Delta_{\mathrm{practice}}
=\widehat{\mathrm{TPS}}_{\mathrm{cen\text{-}ext}}
-\widehat{\mathrm{TPS}}_{\mathrm{comp}}\).  Positive values
mean the censored extension gives a higher score than complete-only
evaluation; negative values mean it gives a lower score.}
\label{tab:robustness_stage2_all}
\end{table}

\noindent\textbf{Overconfident predictor block.}
The completion-token probability predictor is separated because it is the only
non-reference predictor with calibrated mean forecast above \(0.5\).  Its
\(\Delta_{\mathrm{practice}}\) is negative for every score family and weight
schedule, and its magnitude changes little across schedules.  This matches the
artificial-censoring decomposition: under the failure-side branch,
overconfident censored-prefix forecasts are penalized.

\begin{table}[!htbp]
\centering
\small
\setlength{\tabcolsep}{4pt}
\resizebox{\textwidth}{!}{%
\begin{tabular}{llrrrc}
\toprule
Score & Weight schedule &
\(\widehat{\mathrm{TPS}}_{\mathrm{comp}}\) &
\(\widehat{\mathrm{TPS}}_{\mathrm{cen\text{-}ext}}\) &
\(\Delta_{\mathrm{practice}}\) & 95\% CI \\
\midrule
log & linear-front
& \(-2.5255\) & \(-4.5941\) & \(-2.0686\) & \([-2.7092,\;-1.3721]\) \\
log & uniform
& \(-2.5233\) & \(-4.5932\) & \(-2.0700\) & \([-2.7496,\;-1.3661]\) \\
log & exponential-front
& \(-2.5255\) & \(-4.5941\) & \(-2.0686\) & \([-2.8006,\;-1.4609]\) \\
log & linear-back
& \(-2.5255\) & \(-4.5941\) & \(-2.0686\) & \([-2.7303,\;-1.4200]\) \\
\midrule
Brier & linear-front
& \(-0.3091\) & \(-0.4266\) & \(-0.1175\) & \([-0.1590,\;-0.0750]\) \\
Brier & uniform
& \(-0.3081\) & \(-0.4263\) & \(-0.1182\) & \([-0.1615,\;-0.0749]\) \\
Brier & exponential-front
& \(-0.3091\) & \(-0.4266\) & \(-0.1175\) & \([-0.1577,\;-0.0734]\) \\
Brier & linear-back
& \(-0.3091\) & \(-0.4266\) & \(-0.1175\) & \([-0.1622,\;-0.0761]\) \\
\midrule
Beta(2,4) & linear-front
& \(-0.00656\) & \(-0.00897\) & \(-0.00241\) & \([-0.00304,\;-0.00182]\) \\
Beta(2,4) & uniform
& \(-0.00651\) & \(-0.00894\) & \(-0.00243\) & \([-0.00310,\;-0.00183]\) \\
Beta(2,4) & exponential-front
& \(-0.00656\) & \(-0.00897\) & \(-0.00241\) & \([-0.00307,\;-0.00183]\) \\
Beta(2,4) & linear-back
& \(-0.00656\) & \(-0.00897\) & \(-0.00241\) & \([-0.00306,\;-0.00178]\) \\
\bottomrule
\end{tabular}
}
\bottomtablecaption{WebShop natural-censoring robustness for Platt-calibrated
completion-token probability.  This predictor has \(\bar F\approx0.70\) and is
the only non-reference predictor with uniformly negative
\(\Delta_{\mathrm{practice}}\) across score families and weight schedules.}
\label{tab:robustness_stage2_tokenprob}
\end{table}

\noindent\textbf{Sign stability across predictors.}
Table~\ref{tab:robustness_stage2_sign} summarizes the sign of
\(\Delta_{\mathrm{practice}}\) across the five non-reference predictors.
Across the \(5\times3\times4=60\) non-reference predictor--score--weight
cells, 48 are positive and 12 are negative.  All 12 negative cells correspond
to completion-token probability, the only predictor in this sweep with
\(\bar F>0.5\).

\begin{table}[!htbp]
\centering
\small
\begin{tabular}{lrcccc}
\toprule
Predictor &
\(\bar F\) &
linear-front &
uniform &
exponential-front &
linear-back \\
\midrule
Verbal confidence (Platt)          & \(0.38\) & \(+\) & \(+\) & \(+\) & \(+\) \\
Completion entropy conf.\ (Platt)  & \(0.38\) & \(+\) & \(+\) & \(+\) & \(+\) \\
Action-span entropy conf.\ (Platt) & \(0.38\) & \(+\) & \(+\) & \(+\) & \(+\) \\
Action-span token prob.\ (Platt)   & \(0.38\) & \(+\) & \(+\) & \(+\) & \(+\) \\
Completion token prob.\ (Platt)    & \(0.70\) & \(-\) & \(-\) & \(-\) & \(-\) \\
\bottomrule
\end{tabular}
\bottomtablecaption{Sign of \(\Delta_{\mathrm{practice}}\) for non-reference predictors
on WebShop.  Each sign summarizes the corresponding score-family rows under
that weight schedule.  The sign is determined by the predictor's confidence
regime rather than by the choice of score family or weight schedule.}
\label{tab:robustness_stage2_sign}
\end{table}

\noindent\textbf{Margin sensitivity relative to the base-rate reference.}
Table~\ref{tab:robustness_stage2_margin} reports the log-score margin of
verbal confidence over the base-rate reference.  The censored extension
reduces the margin under every weight schedule.  Under exponential-front
weights, the margin changes sign.

\begin{table}[!htbp]
\centering
\small
\begin{tabular}{lrrrr}
\toprule
Weight schedule &
\(\widehat{\mathrm{TPS}}_{\mathrm{comp}}-\widehat{\mathrm{TPS}}_{\mathrm{base}}\) &
\(\widehat{\mathrm{TPS}}_{\mathrm{cen\text{-}ext}}-\widehat{\mathrm{TPS}}_{\mathrm{base}}\) &
Shrinkage &
Rank change? \\
\midrule
linear-front       & \(+0.0123\) & \(+0.0086\)  & \(-30\%\) & No \\
uniform            & \(+0.0301\) & \(+0.0237\)  & \(-21\%\) & No \\
exponential-front  & \(+0.0079\) & \(-0.0025\)  & rank flip & Yes \\
linear-back        & \(+0.0476\) & \(+0.0316\)  & \(-34\%\) & No \\
\bottomrule
\end{tabular}
\bottomtablecaption{Verbal-confidence margin over the base-rate reference under log
score.  The censored extension reduces the margin under every weight schedule;
under exponential-front weights, the margin changes sign.}
\label{tab:robustness_stage2_margin}
\end{table}

\section{Beta-family Scoring Rules}
\label{app:beta}

This appendix records the strict-propriety argument for the beta-family score
used in Section~\ref{sec:prelim}, describes its boundary behavior, and
explains how the parameters \((\alpha,\beta)\) act as cost-shaping choices.
It also notes how the same score can be used as a calibration loss under a
fixed evaluation law.

Throughout this appendix, \(S_{\alpha,\beta}\) denotes the per-step binary
score, while \(\TPS\) denotes its trajectory-level weighted lift.

\subsection{Strict propriety of the beta family}
\label{app:beta:proof}

Recall that the beta-family score is
\[
  S_{\alpha,\beta}(p,1)
  =
  -\int_p^1 c^{\alpha-1}(1-c)^\beta\,dc,
  \qquad
  S_{\alpha,\beta}(p,0)
  =
  -\int_0^p c^\alpha(1-c)^{\beta-1}\,dc .
\]

\begin{proof}[Strict propriety]
The score is defined by the Schervish--Buja threshold-mixture construction
with mixing measure
\[
  \nu(dc)=c^{\alpha-1}(1-c)^{\beta-1}\,dc .
\]
For \(\alpha>0\) and \(\beta>0\), this density is strictly positive on
every open subinterval of \((0,1)\).  Hence the mixing measure has full
support on \((0,1)\).  By Schervish's characterization of binary proper
scores, a threshold-mixture score is strictly proper if and only if the
mixing measure assigns positive mass to every open subinterval of
\((0,1)\).  Therefore \(S_{\alpha,\beta}\) is strictly proper for all
\(\alpha,\beta>0\).

The Brier score corresponds to \(\alpha=\beta=1\), up to positive affine
equivalence:
\[
  S_{1,1}(p,1)=-\frac{(1-p)^2}{2},
  \qquad
  S_{1,1}(p,0)=-\frac{p^2}{2}.
\]
Thus maximizing \(S_{1,1}\) is equivalent to minimizing the usual Brier
loss.

The logarithmic endpoint is obtained as the limiting endpoint
\(\alpha=\beta\to0\).  We treat it as the standard binary log score,
\[
  S_{\log}(p,1)=\log p,
  \qquad
  S_{\log}(p,0)=\log(1-p),
\]
which is strictly proper by direct verification.  If
\(Y\mid\Hcal_t\sim\mathrm{Bernoulli}(q_t)\), then the conditional expected
log score is
\[
  q_t\log p+(1-q_t)\log(1-p).
\]
Its derivative is
\[
  \frac{q_t}{p}-\frac{1-q_t}{1-p},
\]
which vanishes uniquely at \(p=q_t\), and its second derivative is negative
on \((0,1)\).  Hence the log score is strictly proper.
\end{proof}

\subsection{Boundary behavior}
\label{app:beta:boundary}

The log score is unbounded at the wrong boundary:
\[
  S_{\log}(p,1)\to-\infty \quad \text{as } p\to0,
  \qquad
  S_{\log}(p,0)\to-\infty \quad \text{as } p\to1.
\]
In implementations, we therefore clip reported probabilities to
\([\varepsilon,1-\varepsilon]\), with \(\varepsilon=10^{-6}\), before
computing any log score.

Interior beta-family members are bounded on \([0,1]\).  The two wrong-boundary
floors are finite:
\[
  S_{\alpha,\beta}(0,1)
  =
  -\int_0^1 c^{\alpha-1}(1-c)^\beta\,dc
  =
  -B(\alpha,\beta+1),
\]
and
\[
  S_{\alpha,\beta}(1,0)
  =
  -\int_0^1 c^\alpha(1-c)^{\beta-1}\,dc
  =
  -B(\alpha+1,\beta),
\]
where \(B(\cdot,\cdot)\) is the beta function.  Thus, unlike the log score,
interior beta-family members cap the penalty for near-certain wrong
forecasts.

This distinction matters empirically because agentic confidence streams can
be saturated.  Verbal-confidence predictors, for example, often cluster near
\(0.95\) or \(1.00\).  Log score penalizes such saturation sharply when the
trajectory fails, whereas bounded beta-family members saturate at a finite
penalty.  This is why the experiments report log as the canonical unbounded
default, Brier as the bounded symmetric baseline, and Beta(2,4) as a bounded
asymmetric sensitivity check.

The boundedness of interior beta-family members also explains the scope of
the approximate-propriety statement in Remark~\ref{rem:regret}.  The contrast
\[
  D(p)=S_{\alpha,\beta}(p,1)-S_{\alpha,\beta}(p,0)
\]
is bounded for \(\alpha,\beta>0\), but not for the log endpoint.  Hence finite
plug-in error bounds for misspecified censored weights apply cleanly to
interior beta-family members and require additional clipping or boundedness
conditions for log score.

\subsection{Cost-shaping interpretation}
\label{app:beta:cost}

The parameters \((\alpha,\beta)\) do not determine whether the score is
proper.  Strict propriety follows from full support of the mixing measure.
Instead, \((\alpha,\beta)\) determine how the proper score weights different
threshold mistakes.

In the Schervish threshold-mixture representation, each threshold
\(c\in(0,1)\) can be viewed as a cost-sensitive threshold rule. The rule
penalizes a forecast according to whether the reported probability falls on
the wrong side of threshold \(c\).  The beta-family mixing measure
\[
  \nu(dc)\propto c^{\alpha-1}(1-c)^{\beta-1}\,dc
\]
controls how much weight is assigned to each threshold region.

Smaller \(\alpha\) relative to \(\beta\) shifts mass toward lower thresholds.
This places more weight on mistakes associated with reporting success
probabilities above thresholds that the eventual failed trajectory does not
justify.  In agentic settings, this is the overconfident-success regime:
the agent reports a high probability of eventual success, but the trajectory
fails.  Such errors can delay reflection, deferral, or human handoff.

Larger \(\alpha\) relative to \(\beta\) shifts weight toward higher thresholds
and can emphasize underconfidence on successful trajectories.  This may be
appropriate when unnecessary interventions are costly and the system should
be encouraged to proceed when success is likely.  The Brier score,
corresponding to \(\alpha=\beta=1\) up to positive affine equivalence, is the
symmetric bounded case.

Beta(2,4) provides a bounded asymmetric instance of the beta family.  Together
with log and Brier, it lets us compare an unbounded default, a bounded
symmetric score, and a bounded cost-shaped score.

\subsection{Proper scoring rules as calibration losses}
\label{app:beta:loss}

A strictly proper scoring rule can also be used as a calibration-aware loss.
Under the reward orientation used in this paper, the corresponding
trajectory-level loss is the negated score:
\[
  \mathcal L_{\mathrm{traj},\alpha,\beta}(F_{1:T},Y)
  =
  -\sum_{t=1}^{T} w_t\,S_{\alpha,\beta}(F_t,Y).
\]
For complete trajectories, minimizing this expected loss is equivalent to
maximizing the expected \(\TPS\).  By Theorem~\ref{thm:complete}, the unique
population optimum is
\[
  F_t=q_t
  =
  \Pp(Y=1\mid\Hcal_t)
  \quad\text{for every }t.
\]
Thus the same mathematical object can be used either as an evaluator of an
existing uncertainty stream or as a loss for training a confidence head,
post-hoc calibrator, or uncertainty probe under a fixed policy.

This differs from optimizing trajectory ECE.  An optimizer minimizing ECE is
rewarded for marginal calibration within bins, not for the resolution that
proper scoring rules require.  It can therefore fail to recover the full
prefix-conditioned probability process even when binwise calibration error is
small.

The \((\alpha,\beta)\) parameters provide a risk-preference interface for the
loss:
\begin{itemize}
  \item Overconfident success predictions on failed trajectories
        (\(p\to1\), \(Y=0\)) can delay reflection, deferral, or human
        handoff.  A bounded asymmetric member with \(\alpha<\beta\) can
        emphasize this regime while avoiding the unbounded log penalty.

  \item Excessive underconfidence on successful trajectories
        (\(p\to0\), \(Y=1\)) can trigger unnecessary interventions.  A choice
        with \(\alpha>\beta\) can emphasize this regime.

  \item When the two directions are treated symmetrically, Brier
        \((\alpha=\beta=1)\) and log are natural defaults, with Brier bounded
        and log unbounded.
\end{itemize}

\noindent\textbf{Distinction from policy optimization.}
Using \(-\TPS\) as a calibration loss is not the same as using it as a
policy-optimization or RLHF reward.  The propriety theorem fixes the
evaluation law and scores forecasts under that law.  If the policy itself is
optimized against the score, it may change the distribution of histories being
evaluated, conflating policy improvement with calibration.  We therefore use
\(-\TPS\) as a loss for confidence heads, post-hoc calibrators, or uncertainty
probes under a fixed policy, not as a standalone RLHF reward.

\noindent\textbf{Censored trajectories.}
For complete trajectories, \(-\TPS_{\alpha,\beta}\) is a valid calibration
loss for any strictly proper beta-family member.  For censored trajectories,
the exact reduced censored loss requires the continuation-success weight
\(q_Z\):
\[
  q_Z=\Pp(Y=1\mid\Hcal_Z).
\]
When \(q_Z\) is available or consistently estimated, the exact reduced score
in Eq.~\eqref{eq:cen_exact} can be negated and used as the corresponding
censored calibration loss.  When \(q_Z\) is unavailable, the simple
failure-side approximation \(q_Z\approx0\) gives
\(-\TPScensimple_{\alpha,\beta}\).  As
Proposition~\ref{prop:cen_simple_target} shows, this simple loss is proper
for the pseudo-label target
\[
  m_t=\Pp(\Delta=1,Y=1\mid\Hcal_t),
\]
not for the original continuation-success target \(q_t\), unless the missing
success mass is zero.  Thus \(\TPScensimple\) is an operational approximation,
while \(\TPScenexact\) is the strictly proper censored score.

\section{Implementation Details}
\label{app:implementation}

\noindent\textbf{Model and inference.}
All trajectories were collected from \texttt{google/gemma-4-31b-it}
accessed through OpenRouter~\citep{Gemmateam2025gemma4}. Token-level
log-probabilities at every step are the raw materials for the completion-token
and action-span predictors.
Table~\ref{tab:impl_model} summarizes decoding settings;
all values are shared across datasets.

\begin{table}[h]
\centering
\small
\begin{tabular}{lc}
\toprule
Hyperparameter & StrategyQA / Tau2 / HotpotQA / WebShop \\
\midrule
\texttt{temperature}             & 0.2      \\
\texttt{top\_p}                  & 1.0      \\
\texttt{top\_logprobs}           & 5        \\
\texttt{max\_completion\_tokens} & 64{,}000 \\
\texttt{reasoning}               & enabled  \\
\bottomrule
\end{tabular}
\bottomtablecaption{Decoding hyperparameters.}
\label{tab:impl_model}
\end{table}

\noindent\textbf{Agent harness.}
All benchmarks used an AUQ-UAM-style ReAct harness~\citep{Yao2023react}
eliciting four structured fields at every step: \texttt{<think>}, \texttt{<action>},
\texttt{<confidence>}, and \texttt{<explanation>}. The
\texttt{<confidence>} field is the verbal-confidence predictor. Up to two
repair attempts were allowed per step before a parse-error termination was
recorded.

\noindent\textbf{Missing predictor values.}
We do not impute missing per-step forecasts. Malformed or missing structured
confidence fields are handled by the harness retry logic; if parsing still
fails, the episode is assigned stop reason \texttt{parse\_error} and is
excluded before predictor construction. For residual predictor-level
missingness, the raw predictor may record a missing value, but the evaluator
requires a numeric stepwise stream at every scored prefix. A trajectory with
a nonnumeric value is skipped for that predictor row rather than partially
scored or filled with a default. Thus each reported \(\TPS\) row evaluates
complete numeric forecast streams for that predictor; any predictor-specific
sample variation is part of the reported transparency convention.
 
\noindent\textbf{Datasets and run settings.}
Table~\ref{tab:impl_datasets} summarizes per-dataset configurations. The
Tau2 user simulator (Qwen~2.5~7B Instruct) is part of the
benchmark interaction environment; it is not the uncertainty-emitting agent
scored by \(\TPS\). All sampling used seed 42.
Tau2-Bench~\citep{Barres2025tau2bench} uses a \(90\)-task
\texttt{base}-split sample from the airline/retail/telecom domains.
StrategyQA~\citep{Geva2021strategyqa} uses the full
\texttt{wics/strategy-qa} test split. HotpotQA~\citep{Yang2018hotpotqa} uses a
\(2000\)-task sample of the \texttt{fullwiki} validation split.
WebShop~\citep{Yao2023webshop} uses the \(500\)-task fixed evaluation subset.

\begin{table}[h]
\centering
\small
\setlength{\tabcolsep}{5pt}
\begin{tabular}{lrrrrl}
\toprule
Dataset & Tasks & Completed & Max steps & Seed & Retrieval / action interface \\
\midrule
StrategyQA & 2{,}290 & 2{,}229 & 16 & 42 & Local ES (top-$k{=}3$, 420 char obs.) \\
HotpotQA   & 2{,}000 & 1{,}529 &  7 & 42 & Local ES (top-$k{=}5$, 420 char obs.) \\
Tau2-Bench &    90   &   201   & 50 & 42 & Benchmark-native tools; 120-step env.\ cap \\
WebShop    &   500   &   163   & 30 & 42 & Text env.; 1{,}000-product catalog \\
\bottomrule
\end{tabular}
\bottomtablecaption{Per-dataset run configuration. Tau2 reports 201 completed from a
90-task base split (multiple episodes per task). WebShop working sample is 308
(163 completed \(+\) 145 admin-censored) after excluding 192 parse-error
terminations.}
\label{tab:impl_datasets}
\end{table}

\noindent\textbf{System prompts.}
The four prompts below are reproduced verbatim. Across all benchmarks, the
\texttt{<confidence>} field is defined identically: a number in $[0,1]$
representing the probability that the task will eventually be solved
correctly from the current state.

\vspace{1.5em}
\noindent\textit{StrategyQA}

\begin{tcolorbox}[promptbox]
You are a question-answering agent. You solve yes/no questions by searching a Wikipedia-backed corpus.

You may use two tools:\\
1. Search[query]: retrieval from the configured Wikipedia backend. Use it to discover a relevant page/paragraph and load current passage context.\\
2. Lookup[keyword]: local context scan only (no network). It scans the currently loaded passage from the last Search and returns a matching span.

When you have enough information, end with: Finish[yes] or Finish[no]

At every step, use this exact format:

<think>your reasoning about what to do next</think>\\
<action>Search[...] or Lookup[...] or Finish[yes/no]</action>\\
<confidence>0.XX</confidence>\\
<explanation>one sentence explaining your confidence</explanation>

Rules:\\
- confidence is a number between 0.0 and 1.0 representing the probability that this task will eventually be solved correctly from the current state.\\
- Always output all four tags: think, action, confidence, explanation.\\
- Use Lookup only after Search has loaded a passage context.\\
- Prefer Lookup on the current passage before issuing another Search.\\
- The final answer must be exactly yes or no.\\
- explanation must not be empty.\\
- Be concise.
\end{tcolorbox}

\medskip
\noindent\textit{HotpotQA}

\begin{tcolorbox}[promptbox]
You are a question-answering agent. You solve multi-hop questions by searching a Wikipedia-backed corpus.

You may use three tools:\\
1. Search[query]: retrieval from the configured Wikipedia backend. Use it to discover relevant pages/passages and load context.\\
2. Lookup[keyword]: local context scan only (no network). It scans the currently loaded passage and returns a matching span.\\
3. Finish[answer]: terminate with a free-form final answer string.

At every step, use this exact format:

<think>your reasoning about what to do next</think>\\
<action>Search[...] or Lookup[...] or Finish[answer]</action>\\
<confidence>0.XX</confidence>\\
<explanation>one sentence explaining your confidence</explanation>

Rules:\\
- confidence is a number between 0.0 and 1.0 representing the probability that this task will eventually be solved correctly from the current state.\\
- Always output all four tags: think, action, confidence, explanation.\\
- Use concise targeted search queries; multi-hop questions often require evidence from multiple pages.\\
- Use Lookup after Search to extract precise spans from current context.\\
- The final answer should be short and direct.\\
- explanation must not be empty.\\
- Be concise.
\end{tcolorbox}

\medskip
\noindent\textit{Tau2-Bench}

\begin{tcolorbox}[promptbox]
You are a customer-support agent operating in the tau2-bench domain: \{TAU2\_DOMAIN\}.

Follow this domain policy exactly:\\
<domain\_policy>\\
\{TAU2\_POLICY\}\\
</domain\_policy>

You may only use these benchmark-native tools:\\
<tool\_catalog>\\
\{TAU2\_TOOL\_SPECS\}\\
</tool\_catalog>

At every step, output exactly four tags:\\
<think>brief reasoning</think>\\
<action>Tool[tool\_name, \{...valid JSON object...\}] or Respond["..."]</action>\\
<confidence>0.XX</confidence>\\
<explanation>one sentence</explanation>

Rules:\\
- Use Tool[...] only with a listed tool name and a valid JSON object for arguments.\\
- Use Respond["..."] to send a plain-text message to the user.\\
- Output exactly one Action per step.\\
- confidence must be a number between 0.0 and 1.0.\\
- explanation must be a non-empty short sentence explaining your confidence.\\
- Output only these four tags and no extra fields.\\
- Keep think concise and task-focused.
\end{tcolorbox}

\medskip
\noindent\textit{WebShop}

\begin{tcolorbox}[promptbox]
You are solving a WebShop text-environment task from split: \{WEBSHOP\_SPLIT\}.

At every step, output exactly four tags:\\
<think>brief reasoning</think>\\
<action>Search[query] or Click[target]</action>\\
<confidence>0.XX</confidence>\\
<explanation>one non-empty sentence explaining your confidence</explanation>

Rules:\\
- The action must be exactly one of: Search[...] or Click[...].\\
- Use Search[...] to submit a search query when the search bar is available.\\
- Use Click[...] with a target that exactly matches one of the provided clickable targets.\\
- Always emit an <action> tag, even when uncertain.\\
- Do not emit any other action formats (no Tool[...], Respond[...], Finish[...], JSON, or plain text actions).\\
- confidence must be a number between 0.0 and 1.0.\\
- explanation must be non-empty.\\
- Do not output extra tags or extra text outside the four required tags.\\
- Do not repeat search plans or narrate multiple alternatives; choose the single next action.
\end{tcolorbox}

\noindent\textbf{Reproducibility.}
Predictor extraction formulas are in Appendix~\ref{app:transparency};
parse-error exclusion rules and termination categorization are in
Appendix~\ref{app:parse_error}.

\section{Existing Assets and Licenses}
\label{app:assets}

We used only existing model, benchmark, retrieval, and software assets for
evaluation, and we do not introduce a new dataset, model, or benchmark as a
contribution. The language-model trajectories were generated with
Gemma~4~31B accessed through OpenRouter; the hugging face model card
lists the Gemma~4~31B model under the Apache~2.0 license. The
StrategyQA benchmark is used under the MIT license reported by its official
repository. Tau2-Bench is used under the MIT license reported by its
repository. HotpotQA is used under the CC BY-SA~4.0 license reported by the
benchmark website. WebShop is used under the MIT license reported by its
official repository. The Wikipedia text indexed for local retrieval is used
under Wikipedia's content terms, including CC BY-SA~4.0 and GFDL. These
assets are used only for evaluation and retrieval in the reported experiments.

\newpage
\end{document}